\documentclass[10pt]{article}
\usepackage{iclr2027_conference}
\iclrfinalcopy
\usepackage{iftex}
\ifXeTeX
  \usepackage[T1]{fontenc}
\fi
\usepackage{times}

\usepackage{microtype}

\usepackage{amsmath,amsfonts,bm}

\def\eqref#1{equation~\ref{#1}}

\def\1{\bm{1}}

\DeclareMathAlphabet{\mathsfit}{\encodingdefault}{\sfdefault}{m}{sl}
\SetMathAlphabet{\mathsfit}{bold}{\encodingdefault}{\sfdefault}{bx}{n}

\usepackage{tikz}
\usepackage{adjustbox}
\usetikzlibrary{tikzmark,decorations.pathreplacing,backgrounds,calc,positioning}

\usepackage{nicematrix}
\usepackage{wrapfig}
\usepackage{hyperref}
\usepackage{url}
\usepackage{physics}
\usepackage{amssymb}
\usepackage{amsmath}
\usepackage{amsthm}
\usepackage{booktabs}
\usepackage{caption}
\usepackage{cleveref}
\usepackage{graphicx}
\usepackage{subcaption}
\usepackage{colortbl} 
\usepackage{titletoc}

\usepackage{multirow}
\usepackage{graphicx}

\usepackage{algorithm}
\usepackage{algpseudocode}

\let\originalparagraph\paragraph
\renewcommand{\paragraph}[1]{\par\noindent\textbf{#1}\ }

\theoremstyle{plain}
\newtheorem{theorem}{Theorem}[section]
\newtheorem{lemma}[theorem]{Lemma}

\newtheorem{corollary}[theorem]{Corollary}

\theoremstyle{definition}

\newtheorem{axiom}[theorem]{Axiom}

\theoremstyle{remark}

\newtheorem*{theorem*}{Theorem}
\newtheorem*{lemma*}{Lemma}
\newtheorem*{remark*}{Remark}

\usepackage{mdframed}
\mdfsetup{
  skipabove=1pt,
  skipbelow=1pt,
  innertopmargin=4pt,
  innerbottommargin=4pt
}

\makeatletter
\patchcmd{\proof}
  {\topsep6\p@\@plus6\p@}
  {\topsep2\p@\@plus2\p@}
  {}{}
\makeatother

\usepackage{tabularx}
\usepackage{array}
\usepackage{ragged2e}

\definecolor{themePaper}{HTML}{F4F1ED}
\definecolor{themeInk}{HTML}{32292F}
\definecolor{themeAccent}{HTML}{32292F}

\definecolor{themeBlue}{HTML}{379392}
\definecolor{themeGreen}{HTML}{205620}
\definecolor{themeOrange}{HTML}{EE6958}
\definecolor{themeYellow}{HTML}{EDE1D4}
\definecolor{themeYellowAccent}{HTML}{DCFFFD}

\definecolor{themeGrayOne}{HTML}{EEE8E1}
\definecolor{themeGrayTwo}{HTML}{DFD7D0}
\definecolor{themeGrayThree}{HTML}{6C6569}

\definecolor{themeError}{HTML}{EE6958}

\colorlet{evOpen}{themeBlue}
\colorlet{evTraining}{themeOrange}
\colorlet{evFree}{themeGreen}

\definecolor{evSnapshot}{HTML}{7B5AA6}  
\definecolor{evVerify}{HTML}{B04F78}    
\definecolor{evSubmit}{HTML}{D99A32}    

\colorlet{codeblue}{themeBlue}
\colorlet{codegray}{themeGrayThree}
\colorlet{codegreen}{themeGreen}
\colorlet{codebg}{themePaper!20}

\hypersetup{%
    pdftitle={Training Witnesses: Trusting the Training without Trusting the Trainer},
    pdfauthor={Houjun Liu and Pratyusha Sharma},
    pdfborder = {0 0 0},
    colorlinks,
    citecolor=themeBlue,
    linkcolor=themeBlue,
}

\newcolumntype{Y}{%
  >{\RaggedRight\arraybackslash\hspace{0pt}}X%
}

\usepackage{xcolor}
\usepackage{colortbl}
\usepackage{listings}

\newcommand{\eventcell}[2]{%
  \cellcolor{#1!9}#2%
}

\newcommand{\eventbridge}[1]{%
  \makebox[\linewidth][c]{%
    \begin{tikzpicture}[
      x=\dimexpr\linewidth+2\tabcolsep\relax,
      y=1ex,
      baseline=-0.5ex
    ]
      \fill[#1!9]
        (0,-1.30)
        -- (0.13,-0.22)
        -- (0.87,-0.22)
        -- (1,-1.30)
        -- (1, 1.30)
        -- (0.87, 0.22)
        -- (0.13, 0.22)
        -- (0, 1.30)
        -- cycle;
    \end{tikzpicture}%
  }%
}

\newcommand{\route}[3]{%
  \textcolor{#1!90!black}{%
    \textsf{\bfseries #2}\,$\rightarrow$\,\textsf{\bfseries #3}%
  }%
}

\newcommand{\routeiii}[4]{%
  \textcolor{#1!90!black}{%
    \textsf{\bfseries #2}\,$\rightarrow$\,
    \textsf{\bfseries #3}\,$\rightarrow$\,
    \textsf{\bfseries #4}%
  }%
}
\title{Training Witnesses: Trusting the Training without Trusting the Trainer}

\author{Houjun Liu\thanks{Work done while at Microsoft Research.} \\
  Stanford University \\ 
  \texttt{houjun@cs.stanford.edu}
  \And
  Pratyusha Sharma\footnotemark[1] \\
  New York University\\
  \texttt{p.sharma@nyu.edu}
}

\usepackage{xspace}

\newcommand{\systemname}{Witnesses\xspace}
\newcommand{\certificate}{}

\begin{document}
\flushbottom
\everypar{\looseness=-2\relax}
\setlength{\textfloatsep}{6pt plus 2pt minus 2pt}
\setlength{\floatsep}{4pt plus 2pt minus 2pt}
\setlength{\intextsep}{4pt plus 2pt minus 2pt}
\maketitle
\fancyhead{}
\renewcommand{\headrulewidth}{0pt}

\begin{abstract}
Progress in machine learning cannot outpace our ability to verify it. With an explosion in papers today, every scientific claim rests initially on trust in the trainer, leading to uneven evaluation, baselines, and forestalling of reliable progress. Traditionally, the burden of verification falls on the reader, who must reproduce expensive training runs. This strategy is impractical due to an explosion in slop contributions, diversity of methods, and the sheer compute required. We put the burden of proof where it belongs, on the trainer, and in the process also cut the overall cost of verification significantly. We introduce \textbf{\systemname}, a method for certifying training, data usage and evaluation in a neural network training run. Our key insight is that fast behavioral fingerprints with occasional replay challenges are sufficient for auditing neural network training. Our method is applicable at scale with minimal overhead to the trainer, is cheap for the verifier, rejects bad training runs with amplifiable probability, and allows for exact queries of both data inclusion and exclusion. We test our method on language model training runs from 100M to 2B scales, across DDP and FSDP, and demonstrate this minimal overhead. We also introduce a self-regulating leaderboard of ``auto-certified'' training runs that enables shared baselines and progress. We invite the community to participate in the leaderboard to improve reproducibility in machine learning.
\end{abstract}


\section{Introduction}

\begin{figure}[h]
  \centering
    \includegraphics[width=\linewidth]{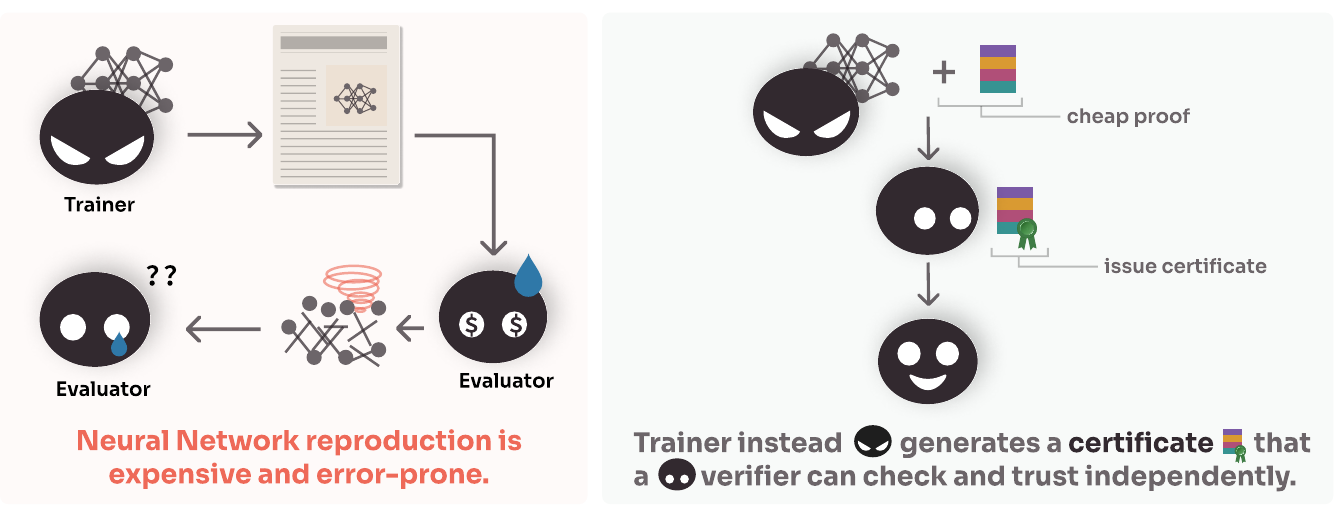}
    \vspace{0.3em}
    \caption{\looseness=-2\relax \textbf{Training Witnesses.} We present a system that creates privacy-preserving \emph{witnesses} that allows a public verifier to cheaply check and trust what information is included in the training of a model without disclosing model training details, enabling trust in the artifacts without trusting the trainer.}
    \label{fig:system-overview}
\end{figure}

An enormous number of papers \citep{su2026selfrankings} now flood machine learning conferences each year, and the number is growing exponentially, including received 60,000 submissions just at ICLR this year. How much of it has moved the field forward?

There's an emerging consensus \citep{shah2026peerreview,baumann2026stop,abid2026icmlreproductions,raff2025reproducible,schaeffer2025refutations,pmlr-v235-rogers24a} in the field that finds it hard to \textit{trust} the papers being written, \textit{reproduce} the results when trust is missing, or even \textit{verify} specific claims made in academic papers. We believe this \textit{progress crisis} \citep{hutson2018artificial} in machine learning is driven by the proliferation of varied baselines, unequal comparisons, and undeclared interventions, rendering standardized comparisons difficult \citep{sculley2025competitions} and thus actual progress less visible. We believe a mere policy change alone cannot fix this alarming crisis. Instead, solving this crisis requires a change in the technology we use to evaluate and certify contributions.

\looseness=-1 Our goal is to mitigate this crisis by \textit{building an algorithmic solution that provides assurance of model developer claims while reducing the cost of verification}. The hardest version of this problem involves certifying entire pretraining runs, which enables standardized comparisons with third-party verification of training without the need to retrain. Reproducibility certificates already exist in the systems research community \citep{usenix2026osdiartifacts}. However, unlike systems, reproduction cannot merely be the responsibility of the program committee since resources required to verify a single ML project are prohibitively higher, let alone a whole batch of projects \citep{aclrollingreview2026letter}. 

In response, \textbf{we introduce an algorithm, a code-base agnostic system, and a self-regulating leaderboard} that solves exactly this hardest setting. We let authors prove their models were trained on the data they state, and in turn enable more confidence that their claims of generalization are indeed from algorithmic contributions. Existing methods to do this move the burden of checking onto a single third-party verifier, which is intractable at current scales especially for research areas such as pretraining, novel architecture, etc. Instead, we leverage the insight that \textbf{a progress certificate enables systematic comparison of any transitions that happened during computation}, including results; our system therefore enables the model developer to publish exactly this evidence of training progress, which can then be selectively revealed to a lower-powered verifier to certify. Our proof is privacy-preserving, meaning the model training technique will not be revealed to the user; undeclared data usage can also be kept private, other than revealing that additional undeclared data was used.

Our method, \textbf{\systemname}, uses the fact that practical data attestation audits can be established cheaply at scale with a fast behavioral fingerprint, allowing secure rare replay challenges. This insight allows us to build a\textbf{ protocol and system that reduce scalable training data attestations to be cheap at scale} incurring only $\leq 1.19 \times$ speed and $\leq 1.11\times$ memory overhead for a 2.02B parameter, 2048 block size FSDP/TP-2 reference training run. We use three insights to achieve this:

\begin{enumerate}
  \setlength{\topsep}{2pt}
  \setlength{\itemsep}{0pt}
  \setlength{\parsep}{0pt}
  \setlength{\partopsep}{0pt}
  \item \textbf{For model parameter provenance}, we introduce and prove that a low-rank, full-spark, orthonormal projection of weight matrices yields a behavioral fingerprint that fixes parameter degrees of freedom. This projection is fast enough to run at every gradient update, and it forces any cheating signal to be non-sparse and consistently detectable.
  \item \textbf{For update function sealing}, we leverage the insight that both Jax and Torch's accelerator serializations are stateless \citep{jax2026export,pytorchxla2026stablehlo,openxla2026compatibility} to ensure all training steps can be exactly replayed without side effects.
  \item \textbf{For efficient implementation}, we combine Deepspeed-style optimizer state offloading \citep{ren2021zerooffload} with a white-box cryptographic sidecar \citep{chow2003whitebox} to build a signed, append-only chain, and produce a cheaply-checkable certificate of training data.
\end{enumerate}

We validate this system between 100M-2B scales, including for sharded and multi-host training, as well as test the method under adversarial optimization. We additionally release \textbf{a web system}\footnote{\url{https://www.dvbench.org/}} that developers can use to attest their training and build flexible leaderboards to display attestations. We hope these techniques and tools can improve the reproducibility crisis in machine learning.

\section{Related Work}
Exact certification of training, data and evaluation is conceptually simple, but practically intractable at scale when done naively. At a high level, there are three overarching goals that such a system has to achieve for widespread deployment: the system has to be cheap to execute, trustworthy despite the low overhead, and meaningfully shift cost away from the verifier (lest it becomes just reproduction).

A first method for certifying proof of training data (PoTD) sequentially attests every training transition by building trust using Zero-Knowledge Succinct Non-Interactive Argument of Knowledge (zk-SNARK) techniques \citep{sun2025zkdl,eisenhofer2025verifiable,garg2023experimenting}. These methods are usually too slow to be practical (with \citet{sun2025zkdl} offering ``less than a second per update for 10M parameters'' performance). Thus, they do not achieve our goal of running with low overhead.
Second, statistical characterization methods exist to remove checkpointing cost \citep{choi2023tools}, leveraging the fact that LLMs temporarily overfit to one sample of training and learn the rough training order of data \citep{kuditipudi2025blackbox,maini2021dataset,choi2023tools}. However, a valid sequence of overfitting to existing data does not preclude additional, undeclared poisoning, leading to low detection rates for data addition. This effect fails our second goal of adding trust.

Thus, more practical PoTD methods often simply write down the training with probabilistic chance. However, such systems still require \citep{madrigalcianci2026probabilistic} an expensive checkpointing and replay procedure, incurring high cost on both the trainer and verifier as shown in Figure~\ref{sketching-insight}. 

Importantly, all approaches are specialized to particular training topologies or algorithms \citep{garg2023experimenting,eisenhofer2025verifiable,sun2025zkdl,maini2021dataset,choi2023tools}, and therefore are unwieldy and do not by themselves generalize across the diverse methods used in LLM development. 


\section{Setup: Training Witnesses}

\begin{wrapfigure}{l}{0.55\textwidth}

\centering
\renewcommand{\arraystretch}{1.10}

\begin{tabular}{@{}c r ccc@{}}
\toprule
& Open: &
\tikzmarknode{openTrain}{\textbf{Training}}
&
\tikzmarknode{openWeight}{\textbf{Weight}}
&
\textbf{API} \\
\midrule

\multirow{6}{*}{%
  \rotatebox[origin=c]{90}{\footnotesize\textbf{Discloses}}%
}
& Data (training)                   & \checkmark &            &            \\
& Data (evaluation)                 & \checkmark & \checkmark &            \\
& Training procedure                & \checkmark &            &            \\
& Hyperparameters          & \checkmark &            &            \\
& Model params & \checkmark & \checkmark &            \\
& Architecture & \checkmark & \checkmark &            \\
& Inference access                  & \checkmark & \checkmark & \checkmark \\

\cmidrule{1-5}

\multirow{2}{*}{%
  \rotatebox[origin=c]{90}{\footnotesize\textbf{Check}}%
}
& $d_{\text{train}}$
    & \checkmark &            &            \\
& $d_{\text{eval}} \notin d_{\text{train}}$
    & \checkmark
    & \tikzmarknode{openWeightBottom}{\checkmark}
    &            \\
\bottomrule
\end{tabular}

\begin{tikzpicture}[remember picture,overlay]

  \coordinate (hlNW) at
    ($(openTrain.north west)+(-\tabcolsep-2pt,8pt)$);

  \coordinate (hlSE) at
    ($(openWeight.east |- openWeightBottom.south)
      +(\tabcolsep+5pt,-20pt)$);

  \coordinate (hlSW) at (hlNW |- hlSE);

  \coordinate (hlMD) at ($(hlSW)!0.5!(hlSE)$);

  \fill[
    themeGreen,
    fill opacity=0.04,
    rounded corners=7pt
  ]
  (hlNW) rectangle (hlSE);

  \node[
    text=themeGreen,
    font=\scriptsize\bfseries,
    anchor=center,
    inner sep=0pt
  ]
  at ($(hlMD)+(0pt,6pt)$)
  {\systemname (ours) can certify \certificate};

\end{tikzpicture}

\vspace{15pt}


\caption{\looseness=-2\relax
\textbf{Disclosure settings and what is certified.} Our method can protect against dataset inclusion
and exclusion attacks in the fully open setting, and data inclusion attacks
in the open-weight setting.}
\label{tab:release-settings}
\end{wrapfigure}

We want a system that can give a verifier confidence that a model was trained by some input stream with no undeclared updates. To do this, we propose a \textbf{training witness certificate}, which is a cryptographically secure, append-only records chain that establishes facts about a training run. The trainer produces a witness so that any verifier can check for any data sample's inclusion and exclusion. We consider two settings in \cref{tab:release-settings}. For \emph{open training}, the verifier has access to model and training data; for \emph{open weight}, the verifier has access to final model only.
We support two queries: 
\paragraph{Membership Queries} Data \textit{membership} queries check if a sample was used in training, defending against \emph{inclusion attacks} where a trainer uses undeclared data (e.g., test-set contamination). 

\paragraph{Exclusion Queries} Data \textit{exclusion} queries check if a sample was not used in training. These queries defend against \emph{exclusion attacks}, in which a trainer falsely claims to have withheld a sample.

\vspace{-1em}
\subsection{Defenses and Assumed Capabilities}
We set up our task as a three-party game as shown in Figure~\ref{boundaries-IP}, with the Trainer (\emph{Alice}, $A$), an untrusted training program; Prover (\emph{Bob}, $B$), an untrusted prover program collocated with the Trainer; and Verifier (\emph{Victor}, $V$), a trusted verifier with access to the full initial training state $(\theta_{0},s_{0})$. 

\paragraph{Attacker Abilities} $A$ is malicious and intends to generate a valid proof tape which contains forged information. We assume that $A$ is a computationally bounded polytime adversary which cannot find collisions in the usual computational security of hash, signature, and commitment schemes. We consider $A$ to have full user-level control over the training process, including the training script, data pipeline, memory, and storage. We assume no boundary between $A$ and $B$, allowing collusion. 

\paragraph{Algorithmic Guarantees} Our tape must be difficult to fabricate, meaning a valid tape that's been modified should become invalid. We also guarantee perfect completeness (all honest $A$ can produce a valid tape) and probabilistically-checkable soundness ($A$ can cheat with probability that's an epsilon-function of the computation that $B$ intends to spend in constructing the tape).

\paragraph{Practical Guarantees} Our formal algorithm shifts the trust of an expensive training procedure to trusting the correctness of cheap hashing steps. To instantiate this trust efficiently, we anchor a white-box crytographic attestor as a part of $B$. Specifically, we ship an obfuscated virtual-machine with our system which measures the rest of the system and derives cryptographic primitives. We make the engineering assumption that $A$ cannot forge the behavior of this VM, cannot lift the code of the VM, or modify the code of this VM without breaking the integrity of the tape. 


\begin{figure}[h]
    \centering
    \includegraphics[width=\linewidth]{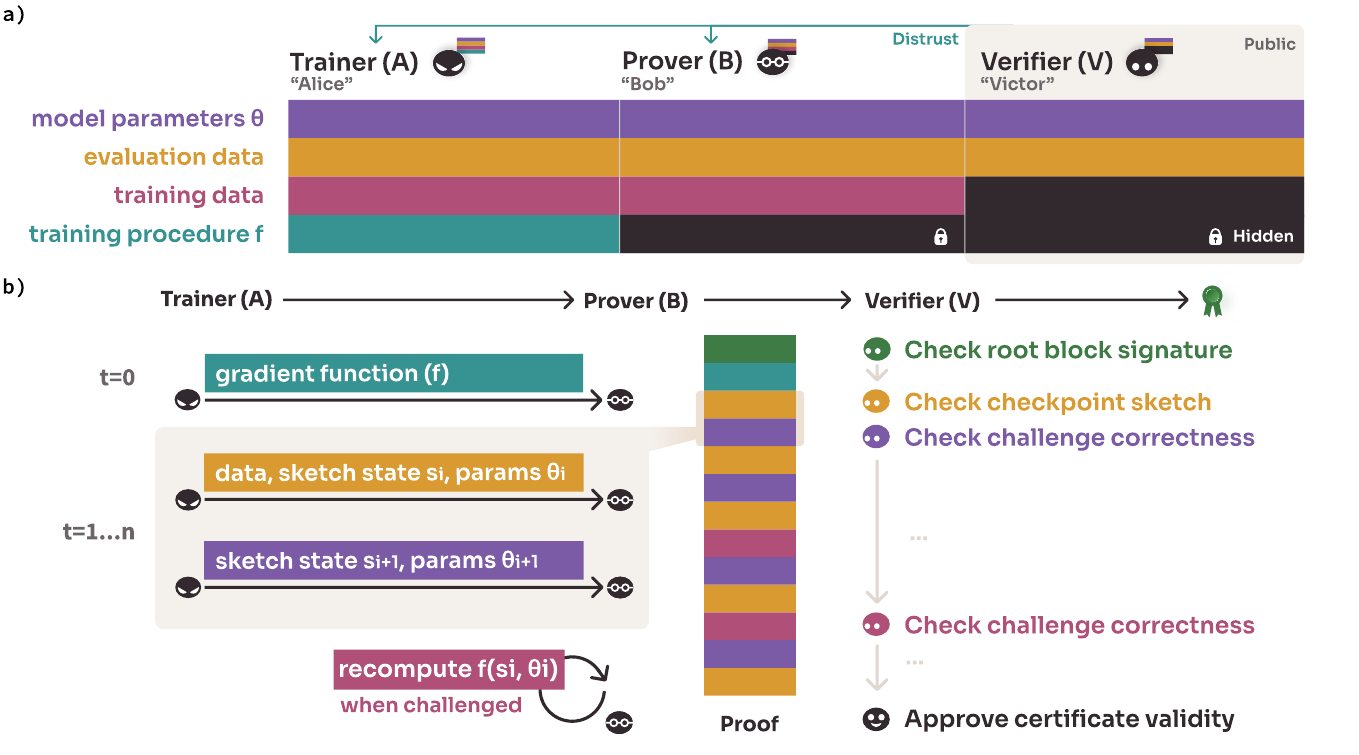}
    \caption{\looseness=-2\relax \emph{(Top)} \textbf{Privacy and trust boundaries between our three interacting parties.} The prover obtains a black-box functional training procedure, and the verifier obtains only a proof of training data with minimal privacy leakage. \emph{(Bottom)} \textbf{Interaction protocol between the parties.} The \systemname system uses a communicating algorithm between a potentially-malicious trainer and prover to generate a privacy-preserving \textbf{public training witness} which any verifier can cheaply check the validity of.}
    \label{boundaries-IP}
\end{figure}
\section{Certification Protocol}
\label{sec:algo}

\paragraph{Tape Structure} We define $\bar{\tau} = \qty(R, M, u_{0}, \dots, u_{n-1})$. The logical tape consists of a root record $R$, a measurement record $M$ ($\tau_{-1}:=\qty(R,M)$), and contains $n$ transitions $u_0 \dots u_{n-1}$ within. We describe the construction of the root record in \cref{appx:sec:root} and describe the measurement record in \cref{appx:sec:measure}. In order to ensure the integrity of the proof tape, we derive the final tape $\tau$ by a cryptographically signed and ratcheted hash-chain. Logically, each link in $\tau_{i} = f_{\text{chain}_{i}}\qty(\tau_{i-1}, \bar{\tau}_{i})$ contains information about the previous record already inserted, and a ratcheted cryptographic signature. We describe this system in further detail in \cref{appx:sec:rachet}. Applying $f_{\text{chain}}$ throughout the tape construction gives the final tape $\tau = \qty(\tau_{0}, \dots, \tau_{n-1})$.

\paragraph{Construction Responsibilities} $A$ is the untrusted trainer; $B_{\text{host}}$ is the host program of the untrusted verifier; and $B_{\text{att}}$ is the trusted attestation engine which is bootstrapped by $B_{\text{host}}$. There are two ways $B_{\text{att}}$ and $B_{\text{host}}$ can be linked. For proofs, we assume the two modules are linked through a cryptographically secure attestation method such as a standard zkVM \citep{risczero2023zkvm}; to speed up our practical implementation, we also introduce an engineering tradeoff which uses whitebox-cryptography \citep{chow2003whitebox} for attestation. Please note, importantly, these choices are consistent with our algorithmic and practical threat models, respectively.

\paragraph{Hashing} We use a series of $\mathcal{H}_{*}$ hash operations throughout which are to be understood as cryptographically ratcheted, stamped hashes. The structure of these hashes is discussed in \cref{sec:commit:hashes}.

\subsection{Update Records}
A logical ``update'' record consists of a series of records which instrument a particular declared transition. The high-level goal is to record the executed transition function $\tilde{f}\qty(\theta_{i}, s_{i}, d_{i}) = \theta_{i+1}, s_{i+1}$ over parameters $\theta_{i}$, optimizer states $s_{i}$, and input data batch $d_{i}$ and check it against a declared $f$.

Update records attest $\tilde{f}$ by committing the output of parameter, state, and data sketch functions $C_{\theta},C_s,C_d$, and check that $\tilde{f} \cong f$, the attested ``true'' update, by probabilistic challenges controlled by $p_{\text{check}}$. Specifically, each update record binds the sketch of the parameter and optimizer state before update $C_{\theta}\qty(\theta_{i}), C_{s}\qty(s_{i})$, the update data itself, $C_{d}\qty (d_{i})$, and the sketch of the parameter and optimizer state after update $C_{\theta}\qty(\theta_{i+1}), C_{s}\qty(s_{i+1})$. We then sample and replay records with probability $p_{\text{check}}$ to check that the transition follows the declared $f$. This record is described in detail in \cref{sec:appx:update}.

\label{sec:update}

\subsection{State Commitments}
\label{sec:sketch}
So far, our design entails effectively transcribing every training step, attesting the transitions cryptographically, and challenging them on occasion. A naive choice of $C_{\theta}$, $C_{s}$, and $C_{d}$ would make the transcript prohibitively long, slow to produce and check. We address this with three sketches:

First, we introduce in \cref{sec:sketch:param} a novel \textit{quasi-geometric} commitment for parameters $C_{\theta}$ which is designed both to attest the identity of the parameters between commitments and to ensure that future verifiers can easily link any checkpoint against a commuted fingerprint geometrically.

Second, we use a combination of a cryptographically secure hash and a Karp-style rolling hash for the data commitment $C_{d}$, which ensures both that downstream verifiers can check any fragment of data used in training and bind the data to the tape. We discuss this in \cref{sec:karp}.

Furthermore, we use a \textit{non-geometric} hash for the optimizer state $C_{s}$ designed to ensure the identity between two distinct commitments of $s_{i}$ (i.e., the one in $\tau_{i}$ and the one in $\tau_{i+1}$) remains the same. We discuss this standard multi-round permutation non-geometric hash in \cref{sec:sketch:optim}.

\subsubsection{Geometric Parameter Sketch}
\label{sec:sketch:param}
Let us first turn to our quasi-geometric parameter commitment $C_\theta$ using orthonormal projections.

\begin{wrapfigure}{l}{0.6\textwidth}
\vspace{-1em}
\begin{minipage}{\linewidth}

\vspace{2pt}
\noindent\rule{\linewidth}{0.8pt}

\captionof{algorithm}{\textsc{SketchCommit }$C_{\theta}(\theta_i, p_{\mathrm{sketch}})$}
\label{alg:sketch-commit}

\vspace{2pt}
\hrule
\vspace{4pt}

\begin{algorithmic}[1]
\Require Parameters
$\theta_i = (\theta_i^{(1)}, \ldots, \theta_i^{(m)})$
\Require Sketch fraction $p_{\mathrm{sketch}} \in (0,1]$
\Require Fixed orthonormal sketch matrices
    $\{Q_j\}_{j=1}^{m}$
\Ensure $C_{\theta}(\theta_i)$

\State $C^{(0)} \gets \bot$
    \Comment{Initial commitment}

\For{$j \gets 1$ \textbf{to} $m$}
    \State $d_j \gets \operatorname{dom}\qty(\theta_i^{(j)})$
    \State $l_j \gets
        \left\lceil p_{\mathrm{sketch}} \cdot d_j \right\rceil$
    \State $l_j \gets \min\{\max\{l_j,1\},d_j\}$

    \Statex
    \State \textbf{assert} $Q_j \in \mathbb{R}^{d_j \times l_j}$ \Comment{Domain is $p_{\text{sketch}}$ portion of leaf}
    \State \textbf{assert} $Q_j^{T}Q_j = I_{l_j}$\Comment{Orthonormality}

    \State $r_i^{(j)} \gets \theta_i^{(j)} Q_j$
        \Comment{Project leaf}

    \State $r_i^{(j)\prime}
        \gets
        \mathcal{H}_{\mathrm{sketch}}
        \qty(
            i,\,
            j,\,
            \operatorname{vec}\qty(r_i^{(j)})
        )$
        \Comment{Hash projection}

    \State $C^{(j)}
        \gets
        \mathcal{H}_{\mathrm{sketch}}
        \qty(
            r_i^{(j)\prime},\,
            C^{(j-1)}
        )$
        \Comment{Merkle fold}
\EndFor

\State \Return $C_{\theta}(\theta_i) \gets C^{(m)}$
\end{algorithmic}\label{sec:alg:sketchcommit}
\end{minipage}
\vspace{4pt}
\noindent\rule{\linewidth}{0.8pt}
\vspace{-2em}
\end{wrapfigure}

\paragraph{Why Just a Low-Rank Projection} We design here a very simple scheme to quickly bind model parameters to tape: just project the weights $\theta_{i}$ against fixed orthonormal vectors. Our key observation, proven in \cref{lem:psb:boundsurviv}, is that modifications to model parameters which survive a random orthonormal matrix $Q_{j}$ cannot be sparse because such random projections are almost surely full spark. Thus, an invalid update either has to touch a lot of coordinates (making it more easily detectable by the challenge in \cref{sec:update}) or has to show up on this sketch.

\paragraph{High-Level Design} The sketch (\cref{sec:alg:sketchcommit}) involves caching projections of orthonormal vectors $Q_{j}$ against each weight matrix $\theta$, thus ``fixing'' degrees of freedom each matrix can express. We then perform a Merkle fold to reduce probes to a single hash.

\subsubsection{Batch Commitments}
\label{sec:karp}

To attest the actual data being used, we also commit hashes of input data $d_i$ during each transition. We use two hashes for this: first, a cryptographically secure hash $C_{d}$ commits the data to ensure sampling is random with respect to data, and a fast Karp-style rolling hash $F_{d}$ attests to the downstream verifier what data is actually being used. Details of both hashes are discussed in \cref{appx:sec:batch-fingerprinting}.


\section{Protocol Guarantees}
\label{sec:security}

\begin{figure}[h]
    \centering
    \includegraphics[width=\linewidth]{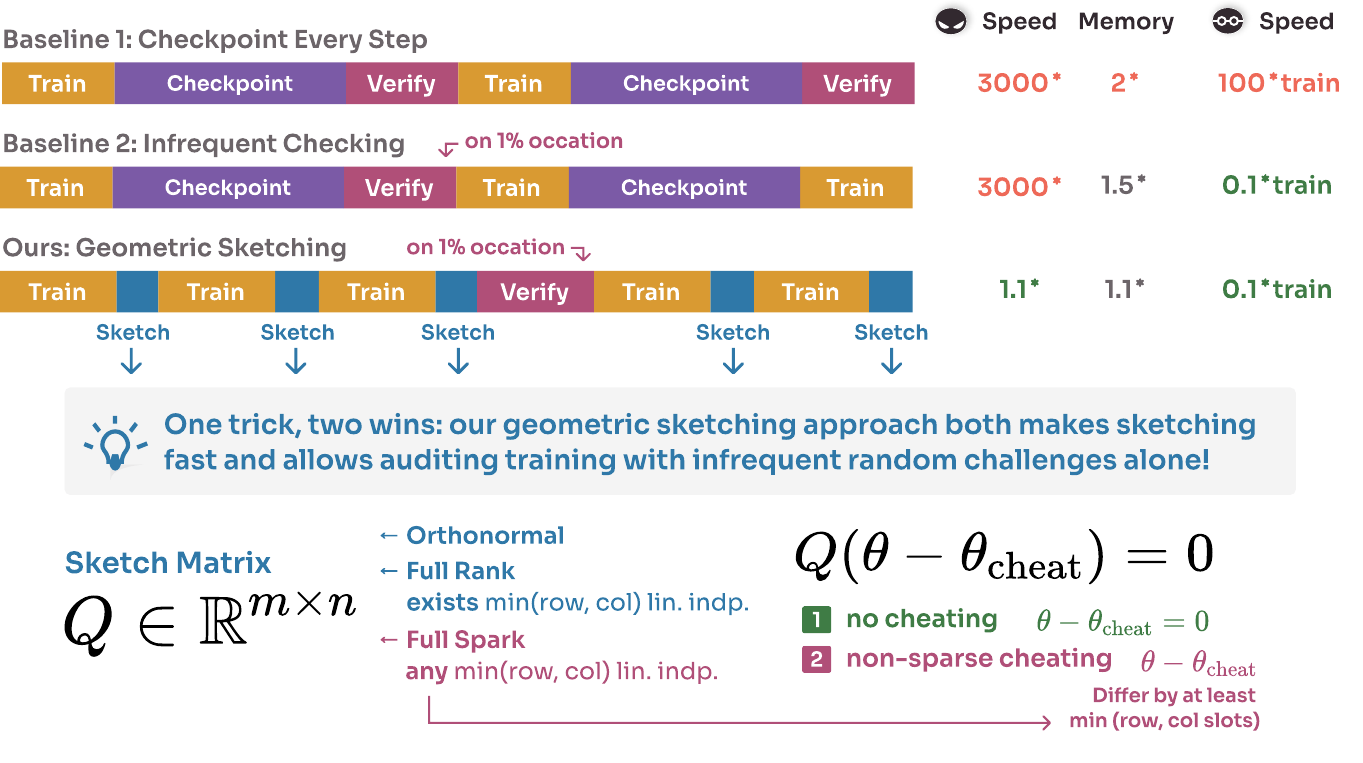}
    \caption{\looseness=-2\relax \textbf{Our protocol's key security insight}. By using a geometric sketch, we can achieve fast binds of our model parameters to a certificate tape while ensuring that cheating which breaks the continuity of the tape cannot be done sparsely. Our protocol is cheap, fast and has a low memory footprint.}
    \label{sketching-insight}
\end{figure}

Here, we sketch out the argument that the protocol given in \cref{sec:algo} admits a probabilistically checkable proof of training data usage. Below we present the proof for both the \emph{soundness} and \emph{completeness} of the protocol. A more detailed proof of soundness can be found in \cref{appx:sec:proofs}. In particular, \cref{appx:sec:threat-model-guarantees} gives the threat model axioms that we rely on. 

We start by defining $I_{f}\qty(\tau)$, the number of invalid transitions recorded in $\tau$ under $f$: $I_{f}\qty(\tau) = \left|\{i \in \{0,\dots,n-1\} \mid \qty(\theta_{i+1}, s_{i+1}) \neq f\qty(\theta_{i}, s_{i}, d_{i})\}\right|$. $\theta_{i}, \theta_{i+1}, s_{i}, s_{i+1}, d_i$ are transitions in $\tau$. We also generate a single-use nonce $\nu$ to be used throughout the proof.

\subsection{Completeness}
We want all valid training programs to produce valid tapes; i.e., we'd like our protocol to be \textit{complete}.

  \begin{theorem}[Completeness of the Protocol]
    \begin{equation}
      \Pr\qty[V\qty(\tau,c,\nu,\theta_0,s_0,\mathcal H_{\text{func}}\qty(f),\theta_f)=1\mid I_f\qty(\tau)=0,\theta_f\in\tau]=1
    \end{equation}
  \end{theorem}
\begin{proof}
  For all $\theta,s$, our sketches are all deterministic functions of their inputs, configuration $c$, and declared protocol randomness, including the shared boundary family, so $x=x' \implies C_{s}\qty(x) = C_{s}\qty(x')$ and $C_{\theta}\qty(x) = C_{\theta}\qty(x')$. The verifier-supplied $(\theta_0,s_0)$ gives the base case. Thus by induction over the tape updates, we have that $C_{s}^{\text{post}}\qty(u_{i}) = C_{s}^{\text{pre}}\qty(u_{i+1})$ and $C_{\theta}^{\text{post}}\qty(u_{i}) = C_{\theta}^{\text{pre}}\qty(u_{i+1})$ for all $i\in\qty{0,\dots,n-2}$. Thus $V_{\text{bind}}, V_{\text{tape}}$ will not reject. Furthermore, since every recorded transition agrees with $f$, $V_{\text{replay}}$ will not reject. Finally, since $\theta_{f} \in \tau$, $V_{\text{bind}}$ will not reject based on our sketch's completeness. Thus $V\qty(\tau,c,\nu,\theta_0,s_0,\mathcal H_{\text{func}}\qty(f),\theta_f)=1$.
\end{proof}

\subsection{Soundness} To ensure that we  prevent cheating, we show that the probability of a successful cheat is bounded:

  \begin{theorem}[Soundness of the Protocol]
    \begin{equation}
      \Pr\qty[V\qty(\tau,c,\nu,\theta_0,s_0,\mathcal H_{\text{func}}\qty(f),\theta_f)=1\mid I_f\qty(\tau)\geq k]\leq\varepsilon_1\qty(k,c)
    \end{equation}
    and
    \begin{equation}
      \Pr\qty[V\qty(\tau,c,\nu,\theta_0,s_0,\mathcal H_{\text{func}}\qty(f),\theta_f)=1\mid I_f\qty(\tau)=0,\theta_f\notin\tau]\leq\varepsilon_2\qty(c)
    \end{equation}
    where $\varepsilon_{1}, \varepsilon_{2}$ are defined as in \cref{prop:psb:boundsurviv-ckp,prop:ckp}, and $k$ is the count of invalid updates.
  \end{theorem}

We give formal proofs of this statement in \cref{prop:psb:boundsurviv-ckp,prop:ckp}. At a high level, we demonstrate this by casework. Specifically, in \cref{lem:bound}, we show that there are only three possible ways of sneaking in invalid updates, each of which has bounded probability of success: by swapping out the optimizer state, by swapping out the sketch, or by swapping out both. 

\section{Efficient Implementation}
\begin{table}[h]
\centering
\scriptsize
\setlength{\tabcolsep}{3pt}
\renewcommand{\arraystretch}{0.78}

\begin{tabularx}{\linewidth}{@{}
  >{\raggedleft\arraybackslash}p{0.065\linewidth}
  Y Y Y
@{}}

\toprule

\textbf{Time}
&
\multicolumn{1}{c}{\textbf{Driver (D)}}
&
\multicolumn{1}{c}{\textbf{Worker (W)}}
&
\multicolumn{1}{c}{\textbf{Broker (B)}}
\\

\midrule


\multicolumn{4}{c}{%
  \textit{flowcontrol ready to admit update $i$}%
}
\\

\midrule


$t_1$
&
\eventcell{evOpen}{%
  Open transaction $i$ (\route{evOpen}{D}{B})%
}
&
\eventbridge{evOpen}
&
\eventcell{evOpen}{%
  Commit data $i$, challenge $i-1$%
}
\\

\addlinespace[0pt]


$t_2$
&
\eventcell{evSnapshot}{%
  Request snapshot (\route{evSnapshot}{D}{W})%
}
&
\eventcell{evSnapshot}{%
  Copy $\theta_{i}, s_{i}$
  (\routeiii{evSnapshot}{W}{A}{W})%
}
&
\tikzmark{packedTop}
\\

\addlinespace[0pt]


$t_3$
&
\eventcell{evTraining}{%
  Request training (\route{evTraining}{D}{W})%
}
&
\eventcell{evTraining}{%
  Compute $\tilde{f}, C_{s}, C_{\theta}$
  (\route{evTraining}{W}{A})%
}
&
\\

&
&
\eventcell{evTraining}{%
  Wait for $\tilde{f}$, $C_{s}$, $C_{\theta}$, donate $\theta_{i}$%
}
&
\\

$t_4$
&
\eventcell{evTraining}{%
  Receive results (\route{evTraining}{W}{D})%
}
&
\eventcell{evTraining}{%
  Copy $\theta_{i+1}, s_{i+1}$
  (\route{evTraining}{A}{D})%
}
&
\tikzmark{packedBottom}
\\

\addlinespace[0pt]


$t_5$
&
\eventcell{evSubmit}{%
  Submit results (\route{evSubmit}{D}{B})%
}
&
\eventbridge{evSubmit}
&
\eventcell{evSubmit}{%
  Sign and write to tape%
}
\\

$t_6$
&
\eventcell{evOpen}{%
  Open transaction $i+1$ (\route{evOpen}{D}{B})%
}
&
\eventbridge{evOpen}
&
\eventcell{evOpen}{%
  Commit data $i+1$, challenge $i$%
}
\\


\midrule

\multicolumn{4}{c}{%
  \textit{if unsampled, asynchronously}%
}
\\

\midrule

$t_7$
&
\eventcell{evFree}{%
  Request memory free (\route{evFree}{D}{W})%
}
&
\eventcell{evFree}{%
  Free $s_{i}$ (\route{evFree}{W}{A})%
}
&
Finalize transaction $i$
\\


\midrule

\multicolumn{4}{c}{%
  \textit{else, asynchronously}%
}
\\

\midrule

$t_7$
&
\eventcell{evVerify}{%
  Request verification (\route{evVerify}{D}{W})%
}
&
\eventcell{evVerify}{%
  Compute $\theta_{i+1}-\delta_{i}$
  (\route{evVerify}{D}{A})%
}
&
\eventcell{evVerify}{%
  Publish challenged index
  (\route{evVerify}{B}{W})%
}
\\

$t_8$
&
&
\eventcell{evVerify}{%
  Wait for verified $\theta_{i+1}^{(j)}$%
}
&
\\

$t_9$
&
\eventcell{evVerify}{%
  Submit checked results (\route{evVerify}{D}{B})%
}
&
\eventcell{evVerify}{%
  Copy $\theta_{i+1}^{(j)}$
  (\route{evVerify}{A}{D})%
}
&
\eventcell{evVerify}{%
  Check, sign, and write to cert.%
}
\\

\addlinespace[0pt]

$t_{10}$
&
\eventcell{evFree}{%
  Request memory free (\route{evFree}{D}{W})%
}
&
\eventcell{evFree}{%
  Free $s_{i}$ (\route{evFree}{W}{A})%
}
&
Finalize transaction $i$
\\

\bottomrule

\end{tabularx}

\begin{tikzpicture}[remember picture,overlay]
  \draw[
    black!40,
    decorate,
    decoration={
      brace,
      amplitude=4pt
    },
    thick
  ]
  ([xshift=-2pt,yshift=1.5ex]pic cs:packedTop)
  --
  ([xshift=-2pt,yshift=-1.5ex]pic cs:packedBottom)
  node[
    midway,
    right=7pt,
    align=left,
    font=\scriptsize\itshape,
    text=black!60
  ] {fully packed,\\pipelined};
\end{tikzpicture}

\vspace{2pt}

\begin{adjustbox}{max width=\linewidth}
\begin{tikzpicture}[
    x=1cm,
    y=1cm,
    lane/.style={
        anchor=east,
        text=themeInk,
        font=\footnotesize
    },
    update/.style={
        anchor=south west,
        text=themeInk,
        font=\scriptsize,
        inner sep=0pt
    },
    tick/.style={
        text=themeGrayThree,
        font=\scriptsize
    },
    block/.style={
        rounded corners=0.7mm
    },
]


\def\timeStep{3.0}       
\def\laneGap{0.5}       
\def\railHeight{0.14}    
\def\smallHeight{0.12}   

\def\labelGap{0.55}      
\def\updateLift{0.06}    
\def\rulerDrop{0.48}     


\coordinate (t0)   at (0,0);
\coordinate (t100) at ($(t0)+(\timeStep,0)$);
\coordinate (t200) at ($(t100)+(\timeStep,0)$);
\coordinate (t300) at ($(t200)+(\timeStep,0)$);
\coordinate (t400) at ($(t300)+(\timeStep,0)$);


\coordinate (broker)      at (t0);
\coordinate (accelerator) at ($(broker)+(0,\laneGap)$);
\coordinate (worker)      at ($(accelerator)+(0,\laneGap)$);
\coordinate (driver)      at ($(worker)+(0,\laneGap)$);


\node[lane] at ($(driver)+(-\labelGap,0)$)
    {Driver};

\node[lane] at ($(worker)+(-\labelGap,0)$)
    {Worker};

\node[lane] at ($(accelerator)+(-\labelGap,0)$)
    {Accelerator};

\node[lane] at ($(broker)+(-\labelGap,0)$)
    {Broker};

%

\coordinate (d0) at ($(driver)+(0.15,0)$);
\coordinate (d1) at ($(driver |- t100)+(0.45,0)$);
\coordinate (d2) at ($(driver |- t200)+(0.45,0)$);

\coordinate (d1) at ($(t100)+(0.45,\laneGap*3)$);
\coordinate (d2) at ($(t200)+(0.45,\laneGap*3)$);

\foreach \p/\lab in {
    d0/$i$,
    d1/$i+1$,
    d2/$i+2$
} {
    \begin{scope}
        \clip[block]
            ($(\p)+(-0.00,-\smallHeight/2)$)
            rectangle
            ($(\p)+(1.10,\smallHeight/2)$);

        \fill[evOpen]
            ($(\p)+(0,-\smallHeight/2)$)
            rectangle
            ($(\p)+(0.55,\smallHeight/2)$);

        \fill[evSnapshot]
            ($(\p)+(0.55,-\smallHeight/2)$)
            rectangle
            ($(\p)+(1.10,\smallHeight/2)$);
    \end{scope}

    \node[update]
        at ($(\p)+(0.15,\smallHeight/2+\updateLift)$)
        {\lab};
}


\coordinate (w0) at (worker);
\coordinate (w1) at ($(t100)+(0.10,\laneGap*2)$);
\coordinate (w2) at ($(t200)+(-0.00,\laneGap*2)$);
\coordinate (w3) at ($(t300)+(0.50,\laneGap*2)$);
\coordinate (w4) at ($(t400)+(-0.10,\laneGap*2)$);

\begin{scope}
    \clip[block]
        ($(w0)+(0,-\railHeight/2)$)
        rectangle
        ($(w4)+(0,\railHeight/2)$);

    \fill[evSnapshot]
        ($(w0)+(0,-\railHeight/2)$)
        rectangle
        ($(w1)+(0,\railHeight/2)$);

    \fill[evSnapshot]
        ($(w1)+(0,-\railHeight/2)$)
        rectangle
        ($(w2)+(0,\railHeight/2)$);

    \fill[evSnapshot]
        ($(w2)+(0,-\railHeight/2)$)
        rectangle
        ($(w3)+(0,\railHeight/2)$);

    \fill[evSnapshot]
        ($(w3)+(0,-\railHeight/2)$)
        rectangle
        ($(w4)+(0,\railHeight/2)$);
\end{scope}

\foreach \p in {w1,w2,w3} {
    \draw[
        themePaper,
        line width=0.8pt
    ]
        ($(\p)+(0,-\railHeight/2)$)
        --
        ($(\p)+(0,\railHeight/2)$);
}

\node[update]
    at ($(w0)+(0.15,\railHeight/2+\updateLift)$)
    {$i$};

\node[update]
    at ($(w1)+(0.15,\railHeight/2+\updateLift)$)
    {$i+1$};

\node[update]
    at ($(w2)+(0.15,\railHeight/2+\updateLift)$)
    {$i+2$};


\coordinate (a0) at (accelerator);
\coordinate (a1) at ($(t100)+(0.80,\laneGap)$);
\coordinate (a2) at ($(t200)+(1.10,\laneGap)$);
\coordinate (a3) at ($(t300)+(1.80,\laneGap)$);
\coordinate (a4) at ($(t400)+(-0.10,\laneGap)$);

\begin{scope}
    \clip[block]
        ($(a0)+(0,-\railHeight/2)$)
        rectangle
        ($(a4)+(0,\railHeight/2)$);

    \fill[evTraining]
        ($(a0)+(0,-\railHeight/2)$)
        rectangle
        ($(a1)+(0,\railHeight/2)$);

    \fill[evTraining]
        ($(a1)+(0,-\railHeight/2)$)
        rectangle
        ($(a2)+(0,\railHeight/2)$);

    \fill[evTraining]
        ($(a2)+(0,-\railHeight/2)$)
        rectangle
        ($(a3)+(0,\railHeight/2)$);

    \fill[evTraining]
        ($(a3)+(0,-\railHeight/2)$)
        rectangle
        ($(a4)+(0,\railHeight/2)$);
\end{scope}

\foreach \p in {a1,a2,a3} {
    \draw[
        themePaper,
        line width=0.8pt
    ]
        ($(\p)+(0,-\railHeight/2)$)
        --
        ($(\p)+(0,\railHeight/2)$);
}

\node[update]
    at ($(a0)+(0.15,\railHeight/2+\updateLift)$)
    {$i$};

\node[update]
    at ($(a1)+(0.15,\railHeight/2+\updateLift)$)
    {$i+1$};

\node[update]
    at ($(a2)+(0.15,\railHeight/2+\updateLift)$)
    {$i+2$};


\coordinate (b0) at ($(t100)+(1.00,0)$);
\coordinate (b1) at ($(t200)+(1.10,0)$);
\coordinate (b2) at ($(t300)+(1.90,0)$);

\foreach \p/\w/\lab in {
    b0/1.10/$i$,
    b1/1.10/$i+1$,
    b2/0.90/$i+2$
} {
    \fill[
        evSubmit,
        block
    ]
        ($(\p)+(0,-\smallHeight/2)$)
        rectangle
        ($(\p)+(\w,\smallHeight/2)$);

    \node[update]
        at ($(\p)+(0.15,\smallHeight/2+\updateLift)$)
        {\lab};
}


\coordinate (ruler0)
    at ($(broker)+(0,-\rulerDrop)$);

\coordinate (ruler100)
    at ($(ruler0)+(\timeStep,0)$);

\coordinate (ruler200)
    at ($(ruler100)+(\timeStep,0)$);

\coordinate (ruler300)
    at ($(ruler200)+(\timeStep,0)$);

\coordinate (ruler400)
    at ($(ruler300)+(\timeStep,0)$);

\draw[
    themeInk!45,
    line width=0.5pt
]
    (ruler0) -- (ruler400);

\foreach \p/\t in {
    ruler0/0,
    ruler100/$\sim\!100$,
    ruler200/$\sim\!200$,
    ruler300/$\sim\!300$,
    ruler400/$\sim\!400$\,ms
} {
    \draw[
        themeInk!45,
        line width=0.5pt
    ]
        ($(\p)+(0,-0.035)$)
        --
        ($(\p)+(0,0.035)$);

    \node[
        tick,
        anchor=north
    ]
        at ($(\p)+(0,-0.07)$)
        {\t};
}

\end{tikzpicture}
\end{adjustbox}

\caption{\looseness=-2\relax Execution timelines of our system; colors represent single logical actions. Python driver (D) coordinates the training process, Rust-side worker (W) dispatches and coordinates training, Rust broker (B) coordinates the transitions and records the output certificate, and Accelerator (A) chip.}

\label{tab:system-timeline}

\end{table}

In addition to the algorithm, we also contribute a low-overhead implementation of the protocol by addressing several challenges of a naive implementation.

\paragraph{Sealing Computation} To guarantee that the attested transition function $f$ is of bounded size, we leverage the insight that modern accelerator compilation already requires computation to be stateless and extracts a closed update program. We thus just extract and replay the lowered accelerator compilation result by either the Jax or Torch libraries. We describe this in detail at \cref{sec:systems:seralize}.

\paragraph{Sealing $B_{\text{att}}$} Other than the whitebox attested binary, many Rust worker operations require dropping back to the Python global interpreter lock to perform an operation (e.g., traverse a PyTree). However, shared Python interpreter state may be mutable by future operations, tampering with the integrity of the attested checker program. To address this, we develop a best-effort sealing mechanism for Python global state through AST monitoring, which we expand on in \cref{sec:systems:invoke}.

\paragraph{Worker Concurrency} Our Python API runs inline to user code, dispatching steps for training and generating artifacts as required by \cref{sec:update}. A Rust worker then separately coordinates the actual cryptography and waits for accelerator work to finish before writing the certificate. To ensure Python and Rust concurrent state transfers are coherent, we design a selective state offloading scheme discussed in \cref{sec:systems:offloading}; we further discuss our broker synchronization logic in \cref{sec:systems:flow-control}.

\section{Experiments} \label{sec:expr}


\paragraph{Certificate Integrity} Under a 100M parameter model, we sweep IID random perturbations to the model parameters, and measure the resulting detection rates. 

\paragraph{Training Data Corruption} We corrupt our training sequences with IID random perturbations across different rates, and measure the probability of detection over multiple independent samples. We compare the emperical rates against our theory, statistical methods \citep{choi2023tools}, and a naive full offload-and-challenge baseline \citep{madrigalcianci2026probabilistic}. We aim to establish that we can catch any effectful sample corruption \citep{souly2025poisoning} at robust rates with strong performance.

\paragraph{Sketch Adversarial Optimization} Finally, since the sketch challenge vector is visible in the certificate, a malicious attacker may produce an entirely honest training and then adversarially optimize against the sketch. To address this, we take a 100M parameter model $\theta$ and choose two sets of 64-token sequences $D^{\text{train}}$ and $D^{\text{mal}}$. We use AdamW to minimize:
\begin{equation} \label{eqn:adv}
  \mathcal{L}(\theta;\lambda)
  =
  D_{\mathrm{KL}}^{\mathrm{train}}(\theta_0\Vert\theta)
  -\lambda D_{\mathrm{KL}}^{\mathrm{mal}}(\theta_0\Vert\theta)
  +
  \operatorname{RMS}\!\left(
  \frac{S(\theta)-S(\theta_0)}
       {a_{\mathrm{tol}}+r_{\mathrm{tol}}|S(\theta_0)|}
  \right),
\end{equation} 
where \(S\) is the committed parameter sketch, $\lambda$ is a Pareto tradeoff parameter, $a_{\text{tol}}=1\times 10^{-8}$ and $r_{\text{tol}}=1\times 10^{-5}$ are calibration values for \texttt{float32} precision. In addition to optimizing this loss, we additionally perform the optimization in a subspace orthogonal to the sketch. That is, for $Q$ being the orthonormal basis of the sketch, we update parameters as $\Delta_{\text{attack}} = \Delta_{\text{naive}_{\mathcal{L}}}\qty(I - QQ^{T})$. 

\paragraph{Performance Evaluations} Our primary performance evaluations involve pretraining a billion-param scale dense GPT model with a weight-tied embedding of 100,288 tokens, a significant increase over published baselines. We profile four regimes: 1) 1.03B parameter GPT model under data-parallel training (DDP) across 2 GPUs; 2) 2.02B parameter GPT model under two-way tensor-parallel training (TP-2) across 2 GPUs; 3) roofline analysis of our system under one GPU, training the same 1.03B parameter GPT model as in the DDP case; 4) two-GPU TP-2 scaling evaluation at various scales ranging from 100M parameters to 2B parameters.

\section{Results}\label{sec:res}

\begin{figure}[!htb]
  \centering
  \includegraphics[width=0.83\linewidth,trim=6 7 4 5,clip]{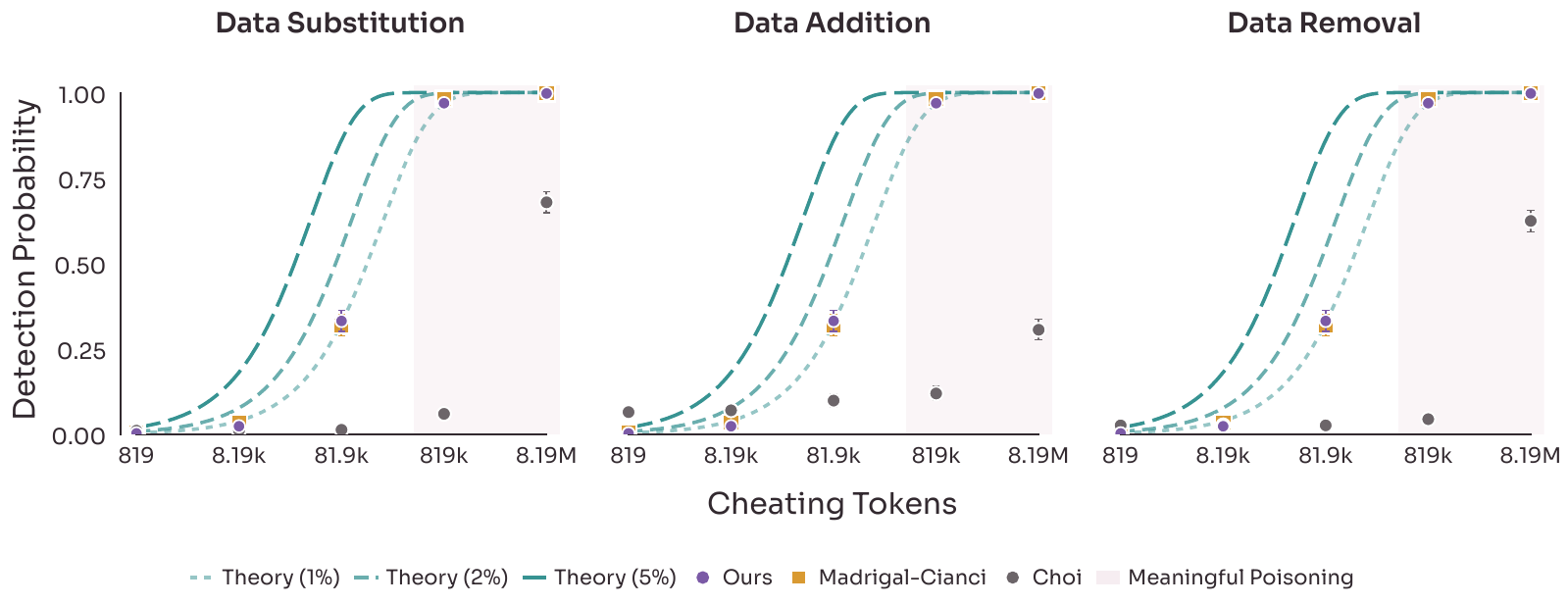}
\caption{\looseness=-2\relax \textbf{Our model achieves near-certain detection rates of data inclusion, exclusion, and substitution attacks at any amounts that would constitute meaningful poisoning}. End-to-end data uinsertion attacks, sweeping across number of tokens inserted. Empirical checks performed during $1\%$ of updates, in line with all performance evaluations. \emph{(Left)}: attack by substituting declared and undeclared sequences. \emph{(Middle)}: attack by adding sequences. \emph{(Right)}: attack by subtracting sequences. Attacks distributed uniformly. Shaded region represents literature-established rates of effective poisoning attacks; error bars are 10-90 Beta posterior credible intervals.}
\label{fig:sec:data-sub}
\end{figure}


\begin{figure}[!htb]
  \centering
  \begin{minipage}{0.48\linewidth}
    \centering
    \includegraphics[width=0.83\linewidth,trim=12 0 7 13,clip]{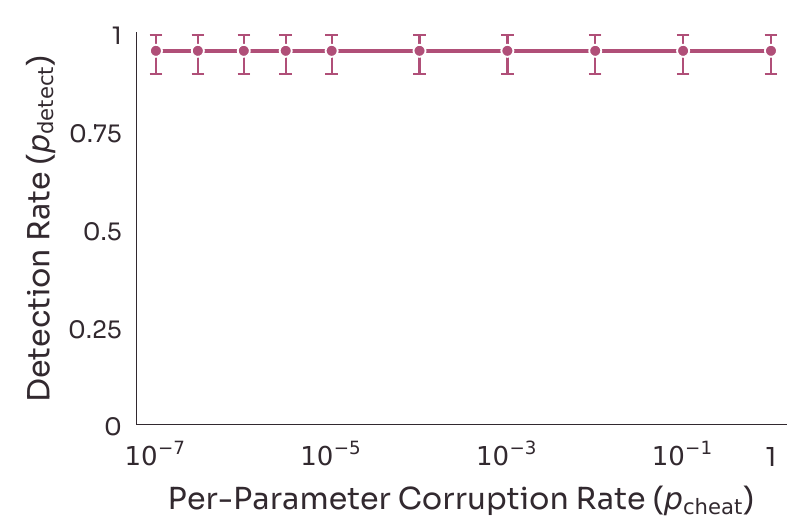}
    \caption{\looseness=-2\relax \textbf{Swapping out parameters uniformly is detected with near certainty}. Certificate rejection rate as a function of per-parameter invalid transition probability at 100M params. Error bars are 10-90 Beta posterior credible intervals.}
\label{fig:sec:parameter-corruption}
  \end{minipage}
  \hfill
  \begin{minipage}{0.48\linewidth}
    \centering
    \includegraphics[width=0.83\linewidth,trim=5 3 2 9,clip]{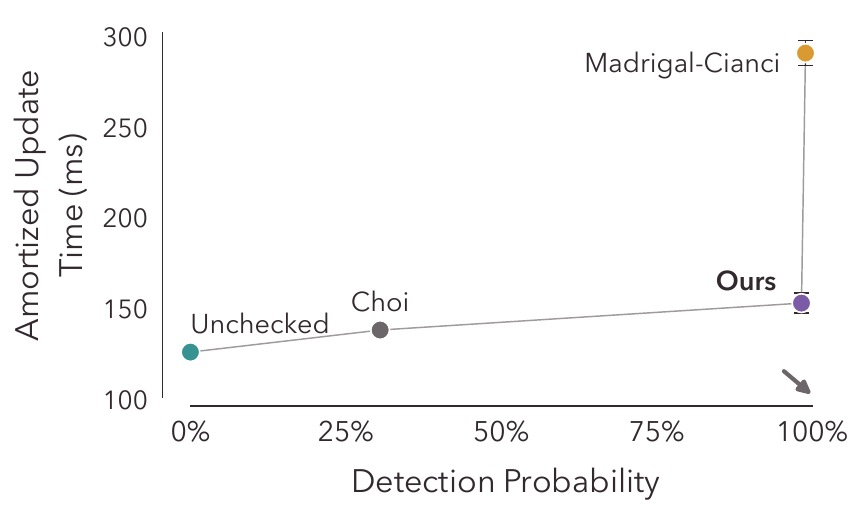}
    \caption{\looseness=-2\relax \textbf{Our method achieves low overheads versus baselines}. Per-step single GPU training overheads versus detection rates on 100M dense GPT-style model on single-GPU with sizes scaled for roofline saturation. Gray arrow shows direction of improvement. NVIDIA A100 PCIe.}
\label{fig:sec:baseline-costs}
  \end{minipage}
\end{figure}

\begin{figure}[!htb]
  \centering
  \makebox[\linewidth][c]{%
    \includegraphics[width=0.405\linewidth,trim=12 11 14 9,clip]{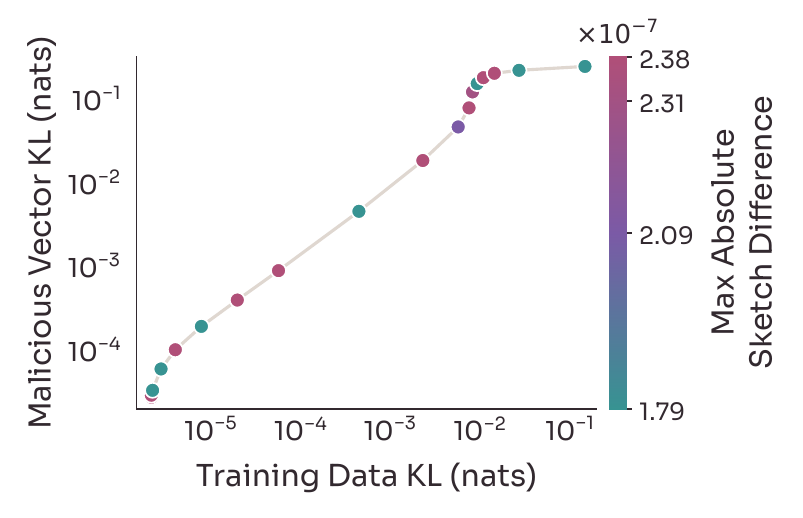}%
    \hspace{0.04\linewidth}%
    \includegraphics[width=0.405\linewidth,trim=12 11 14 9,clip]{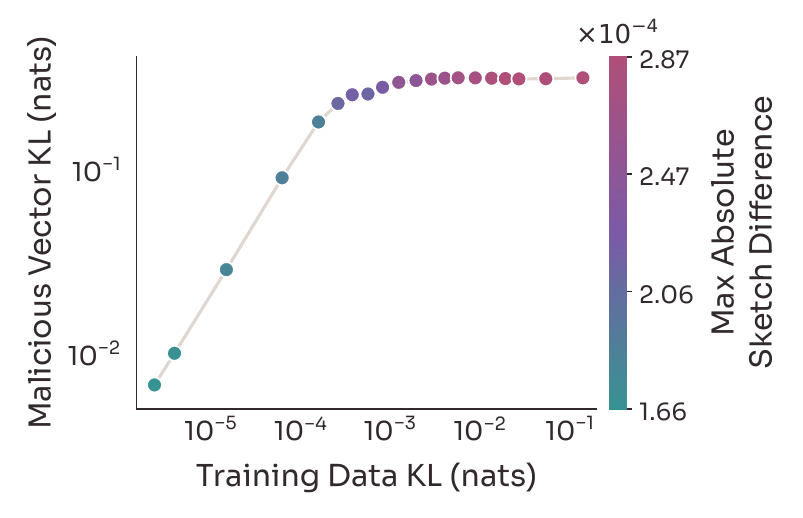}%
  }
  \caption{\looseness=-2\relax \textbf{Even active optimization against our sketch still cannot evade it without meaningfully shifting the model towards training distribution collapse.} We sweep across a Pareto tradeoff parameter $\lambda$ to minimize \cref{eqn:adv}. \textbf{Left}: directly minimizing \cref{eqn:adv} with sketch-nullspace projection. \textbf{Right}: attempting to modify output on malicious KL without sketch error minimization or null space projection. All plotted points detected by sketch at floating point precision $r_{\text{tol}}=1\times 10^{-5}$, $a_{\text{tol}} = 1\times 10^{-8}$.}
\label{fig:sec:optimized-attack}
\end{figure}

\paragraph{Our system achieves strong performance with theoretically-sound detection rates.} In \cref{fig:sec:data-sub}, we show that our end-to-end system can detect all poisoning-regime data substitution attacks with near-certainty, even while checking $1\%$ of transitions. Our system outperforms \citet{choi2023tools} in detection rates, while performing dramatically faster (\cref{fig:sec:baseline-costs}) than offload-and-challenge \citet{madrigalcianci2026probabilistic}. Our detection rates match theoretically estimated detection almost exactly.

\paragraph{Unlike prior work, our system reliably detects both inclusion and exclusion attacks.} Although \citet{choi2023tools} is reasonably fast, a provenance-style statistical check cannot reliably detect data addition attacks \cref{fig:sec:data-sub} (middle). Our system reliably detects both inclusion and exclusion attacks.

\paragraph{Parameter corruption is detected with certainty.} As seen in \cref{fig:sec:parameter-corruption}, our sketch is able to detect in practice any dependently distributed random parameter corruption with almost certainty. In fact, structured random parameter corruption in select subspaces is also reliably detected \cref{fig:sec:structured-attack}.

\paragraph{Strong detection rate persists even through adversarial optimization.} Even despite performing adversarial optimization against our publicly-disclosed sketch, \cref{fig:sec:optimized-attack} highlights that our system is still able to detect any meaningful perturbation to the model outputs. Minimizing the sketch error meaningfully reduces the ability to achieve high KL divergence on malicious samples while maintaining low KL divergence on training samples. Our checkpoint provenance scheme, even despite whitebox access to the sketch, is therefore still secure against trivial offline nullspace attacks.

\begin{figure}[!t]
  \centering
  \makebox[\linewidth][c]{%
    \includegraphics[width=0.435\linewidth,trim=5 4 2 5,clip]{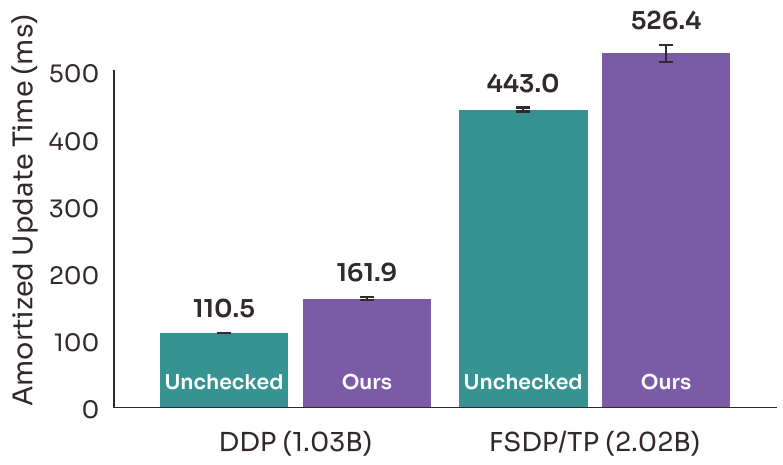}%
    \hspace{0.04\linewidth}%
    \includegraphics[width=0.435\linewidth,trim=5 4 2 5,clip]{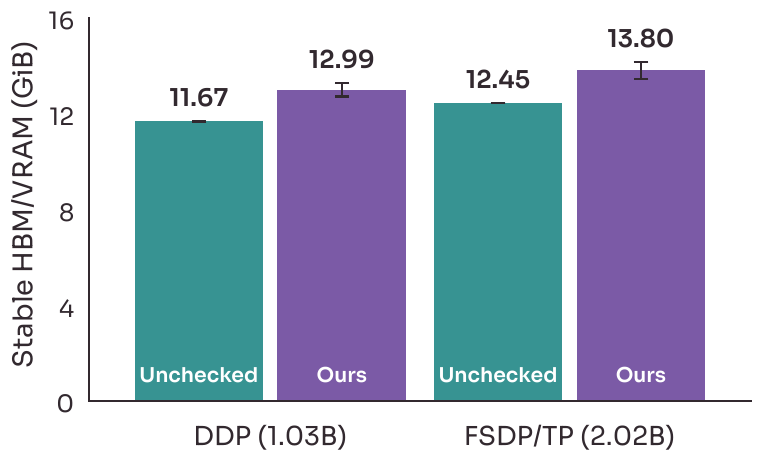}%
  }
\caption{\looseness=-2\relax \textbf{Our method achieves remarkably low overhead at sharded training settings, where it is communication-bound}. Communication-bound overheads. \textbf{Left}: amortized per-step training time; \textbf{Right}: mean HBM resident size (bytes in use) per GPU. Error bars from Student's t-test confidence intervals over 100M tokens of training at $1\%$ check probability. DDP and FSDP sharding conducted over 2 NVIDIA H100 NVL72.}
\label{fig:sec:performance}
\end{figure}

\begin{figure}[!t]
  \centering
  \makebox[\linewidth][c]{%
    \includegraphics[width=0.435\linewidth,trim=11 0 10 8,clip]{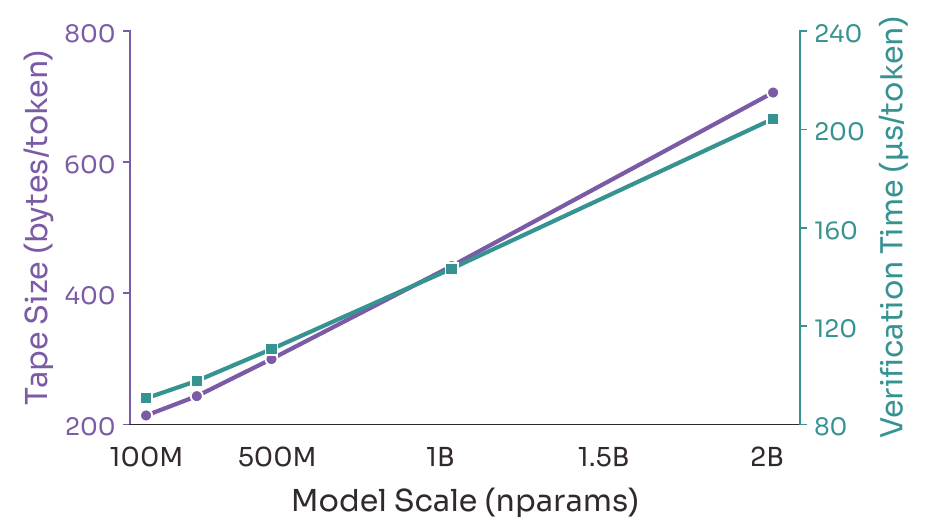}%
    \hspace{0.04\linewidth}%
    \includegraphics[width=0.435\linewidth,trim=20 0 10 11,clip]{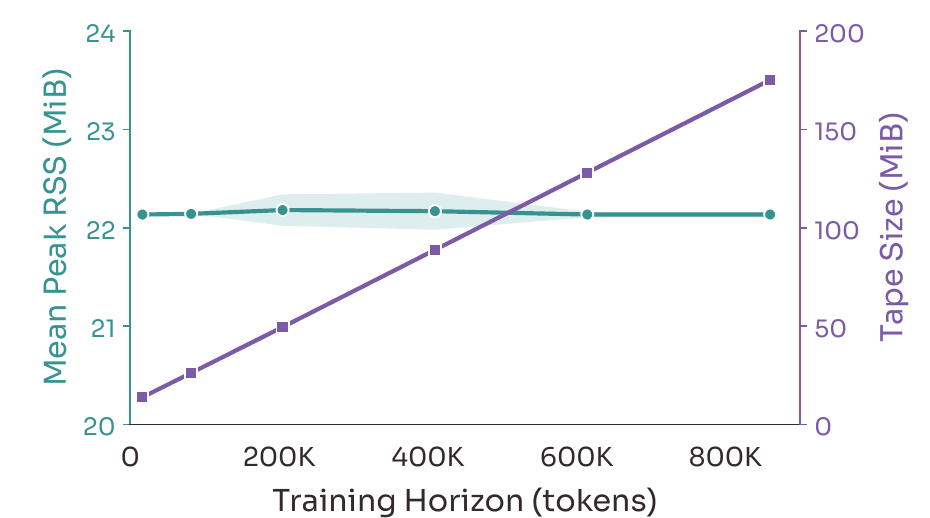}%
  }
  \caption{\looseness=-2\relax \textbf{Our verification system maintains constant memory (RSS) and linear tape size to model size, meaning it can be effectively streamed}. \textbf{Left}: Certificate length and verification time per token across a variety of scales. \textbf{Right}: verification procedure Resident Set Size (RSS) as a function of cert. length. Verification conducted using \texttt{wait4} syscall under a single pinned physical core of an AMD EPYC 9554 processor. Error bars from Student's t-test confidence intervals.}
  \label{fig:sec:verification}
\end{figure}

\paragraph{Our system performs most strongly in sharded, communication-bound settings.} Our headline performance results, $\leq 1.19\times$ speed overhead, are achieved in \cref{fig:sec:performance}, where we are able to elide copying by overlapping them with communication. Saturation experiments in \cref{fig:sec:single-gpu} show that we achieve this without sacrificing the MFU or throughput of underlying models.

\paragraph{Verification is streamable with constant working set.} One tradeoff of our system is that the certificate grows log-log linearly with model parameters. However, \cref{fig:sec:verification} shows that our verification procedure only needs to keep a constant working set of the certificate in memory. A large certificate can thus be effectively verified online in a streaming fashion without ever storing the whole certificate.

\section{Discussion}
In this work, we introduced \systemname, a practical system for attesting training progress of large language models which detects data attacks across 100M-2B parameter GPT-style models trained with both DDP and FSDP. In practice, \systemname is the first attestation system of model training progress demonstrated at billion-scale LLM training with practical overheads, making it immediately applicable at academic and open-weight scales.

Although our current solution extends solely to the process of training, it is general over any closed update function of parameters, and fast enough to deploy at scale. Future work can extend this to support additional forms of (arguably less latency sensitive) workflows. Thus, we believe \systemname opens an exciting new field of work that provides technical governance and integrity to scientific results previously not possible at modern ML-scales. To encourage the adoption, we release \systemname as a leaderboard, described in \cref{sec:leaderboard}, allowing submissions of training runs to be attested and verified by the community. We hope that these tools will be able to meaningfully make model training more reproducible, leading to consistent progress in our exponentially growing field.

\subsection*{Acknowledgments}
We thank our friends and colleagues at Microsoft Research, NYU, and Stanford for their support during this project. Specifically, we thank John Langford, Eshaan Nichani, and Nived Rajaraman at Microsoft for their feedback throughout development. We appreciate the interesting conversations, testing, and feedback from Shikhar Murty, Kyunghyun Cho, Hyojin Bahng, Xi Wang, and Kaden Zheng. Lastly, we thank Chris Manning, Diyi Yang, Dawson Engler, Joseph Shetaye, Amelia Hardy, Tianle Yu, Huxley Marvit, Albert Huang, and Zachary Sayyah for their feedback on drafts of this work. Houjun completed this work as an intern at Microsoft Research, and Pratyusha was a Senior Researcher at Microsoft.



\bibliography{citations}
\bibliographystyle{iclr2027_conference}

\clearpage
\appendix
\let\paragraph\originalparagraph
\startcontents[appendices]
\begingroup
\hypersetup{hidelinks}
{\Large\bfseries Appendix\par}
\vspace{1.5em}
{\large\bfseries Table of Contents\par}
\vspace{0.5em}
\titlecontents{section}
  [2em]
  {\small\bfseries}
  {\contentslabel{2em}}
  {}
  {\hfill\contentspage}
\titlecontents{subsection}
  [4em]
  {\small}
  {\contentslabel{3em}}
  {}
  {\titlerule*[0.5pc]{.}\contentspage}
\printcontents[appendices]{}{1}{\setcounter{tocdepth}{2}}
\endgroup
\clearpage
\section{Related Work}
\label{appx:sec:related-work}

\paragraph{Statistical Training Data Provenance} Our work traces its lineage to systems that check what data is being used during training. Usually, these systems can only establish positive inclusion evidence (``is data X used?''). Methods include using memorization as an inclusion heuristic \citep{choi2023tools}, post-hoc checkpoint evidence \citep{maini2021dataset}, using training order and forgetting dynamics \citep{kuditipudi2025blackbox,zhu2025independence}. Our work follows the same goal as proving training data usage, but also introduces the ability to prove negative evidence (``is data X not used?'') and provides formally guaranteed properties.

\paragraph{Cryptographic Proof of Training} Cryptographic proof of training methods take the opposite tradeoff as statistical ones; instead of checking for evidence of inclusion, these methods certify the process of training using expensive formal proofs. Methods include training computation certificates restricted to tiny architectures \citep{garg2023experimenting,eisenhofer2025verifiable,sun2025zkdl}, zero-knowledge proof systems for fully dense networks \citep{abbaszadeh2024zero}, and formal sketches of training certification with dedicated hardware \citep{peigne2026zero}. We use a probabilistically sampled method, most similar to \citet{madrigalcianci2026probabilistic} but with a fast provenance sketch that dramatically reduces the cost of verification while maintaining amplifiable soundness.

\paragraph{Fast Fingerprints of Training State} A key part of our method involves fast per-step checkpoint evidence, which conventionally performs slowly and is the key drawback of methods like \citet{jia2021proof}. Faster sketch methods include classical rolling hashes \citep{karp1987efficient}, randomized, feature-preserving dimensionality reduction \citep{weinberger2009feature}, and newer decision boundary characterizations such as \citet{yang2026fingerprinting} or data deletion sketches such as \citet{gunn2026sketch}. Our method uses low-rank orthonormal state projections to cheaply constrain degrees of freedom, leading the sketch to be performant enough to run during every update.

\paragraph{Sampled Probabilistic Replay Algorithms} Our method derives its amplifier assurance using a conventional assurance technique of commit and random-check. This technique traces its lineage to classical random verification literature \citep{freivalds1979fast}, which evolved into online partial checking systems for coherence such as \citep{blum1991checking}. \citet{jia2021proof} has also recently used this technique to verify machine learning training, but its reliance on fully streaming frequent chunks of parameter updates makes it impractical at scale.

\paragraph{Fast Parameter Offloading} Finally, to make our system practically feasible, we borrow from the literature of fast parameter streaming. This line of work is most notably shown by the DeepSpeed project in the ZeRO algorithm \citep{rajbhandari2020zero}. Specific methods include eliding GPU memory movement during all gathers \citep{ren2021zerooffload}, offloading only learned sparse projectors \citep{chen2025offloading}, and prioritizing asynchrony using gradient information \citep{lan2025zenflow}. Our work also uses elided offloading, but makes the realization that waiting for offloading to finish is only important during sampled replay steps and thus makes no attempts at intermediate coherence.


\section{The \systemname Leaderboard System}\label{sec:leaderboard}

\subsection{Overall Design}
\label{appx:sec:leaderboard-design}
Our leaderboard system allows users to stream their training runs to a publicly trusted verifier website, which cheaply verifies the training run and queries any data hashes in the streamed certificate against a centralized repository of corpora. If a given sample is found, the website additionally displays the provenance of the sample, percentages of identified text and the corresponding corpus name.

Measurement results from evaluation can be attested in a similar manner, where inputs and outputs of the evaluation are serialized and attested as declared ``update functions'', but verification of the scoring procedure is not performed by the leaderboard. This design allows for arbitrary evaluation processes, including ones that require an environment, to be used.

\begin{figure}
  \centering
  \includegraphics[width=\linewidth]{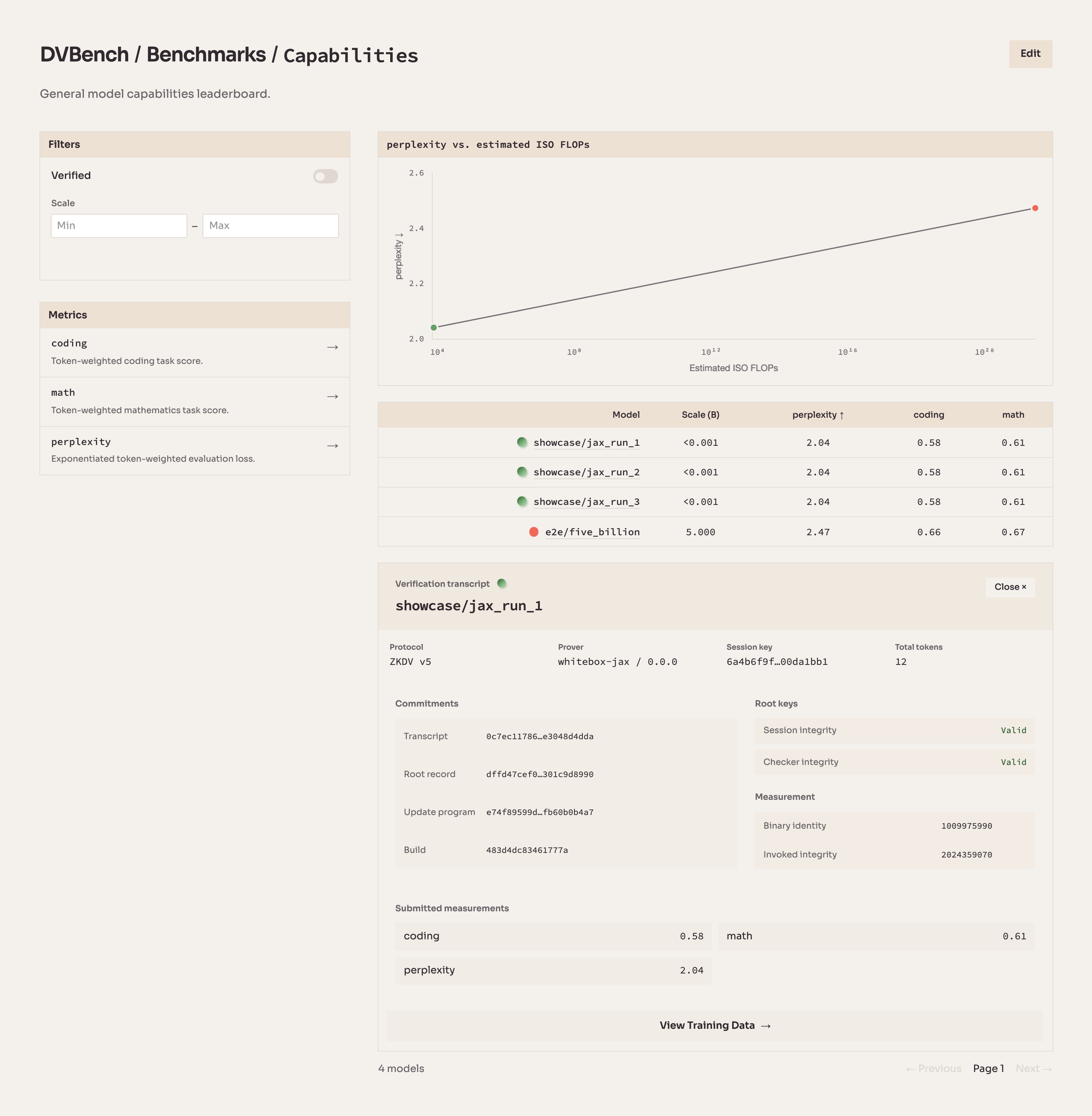}
  \caption{Verification transcript of a training run showing the metadata available to the users. }
  \label{fig:leaderboard-1}
\end{figure}

\begin{figure}
  \centering
  \includegraphics[width=\linewidth]{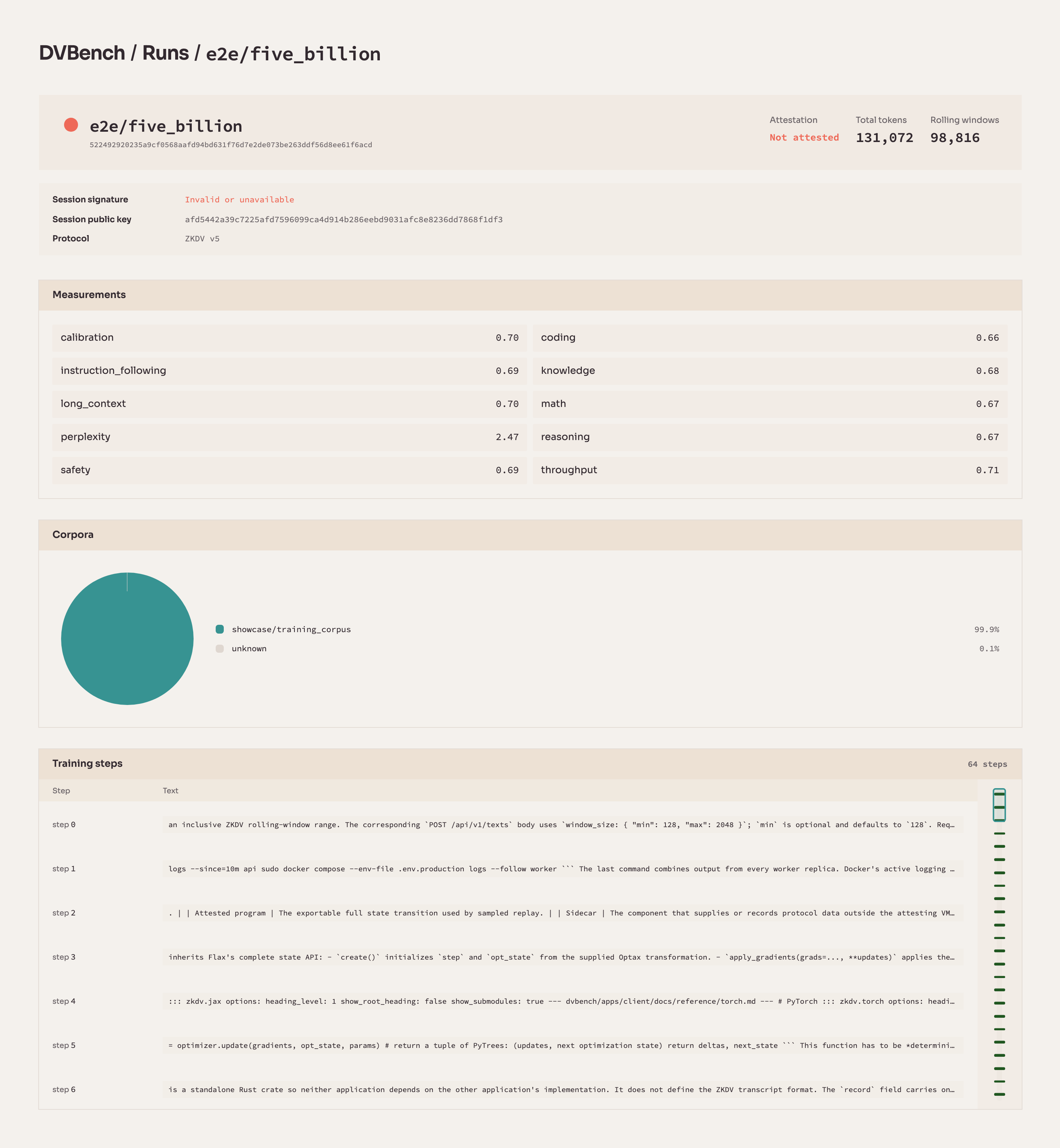}
  \caption{Detail screen of a run with training data attested and revealed; unverified runs can also be listed on our leaderboard relying on the traditional method of trusting the trainer.}
  \label{fig:leaderboard-2}
\end{figure}

\subsection{Implementation Details}
\label{appx:sec:leaderboard-implementation}
We leverage a three-layer architecture for our leaderboard system.

\paragraph{Client Streaming} Client transcript information (each link in the hash chain) is streamed to the verifier using WebSockets, allowing for low-latency streaming of the certificate. The API layer receiving this stream is stateless, and will immediately inject the stream into an ephemeral NATs JetStream \footnote{https://nats.io/} queue for processing.

\paragraph{Distributed Workers} Each worker holds a temporary lease on a particular transcript, instantiating a stateful verifier that pulls off of the JetStream queue and verifies the transcript. Due to our verification algorithm's small work set requirement \cref{fig:sec:verification}, we can run many verifiers in parallel and quickly serialize the verification state should trainers disconnect or workers crash.

\paragraph{Long-Term Database} Upon verification completion, the verified state is signed and stored in a long-term database. To facilitate quickly querying dataset inclusion, we fold all of the training data hashes into a bloom filter that is stored alongside the verified transcript. The web interface then queries this database for displaying the provenance of the training run.

\subsection{User Experience}
\label{appx:sec:leaderboard-user-experience}
Despite the relative complexity of our system, \systemname does not actually require that many user touchpoints to use. Specifically, our algorithm only requires that the users annotate their attested function which we serialize, and wrap their training step so we can pipeline it into the attestation system. To this end, we leverage Maturin \footnote{https://github.com/pyo3/maturin} to design a simple Python facade that allows easy integration of \systemname into existing training harnesses. Below we will outline the application of this system in user code for the Torch variant of our implementation.

\paragraph{Connect Streaming Client} Due to the relatively larger size of the proof certificate, we never store the certificate on training work and instead stream to the verifier directly. This design can also be modified to store the certificate on a local disk without changing the verification procedure. Thus, users first connect to a streaming client:

\begin{lstlisting}[language=python]
from dvbench import Client
client = Client(ENDPOINT)
experiment = client.experiment("owner/model")
proof = experiment.proof(backend="torch")
proof.start(rolling_window=32)
\end{lstlisting}

\paragraph{Wrap Model Prototypes} We apply our facade to \texttt{nn.Module} prototypes by wrapping the model and optimizer with our attestation system.

\begin{lstlisting}[language=python]
model, optimizer = proof.prepare(model, optimizer)
\end{lstlisting}

\paragraph{Attest Training Step} We then wrap the training step with our facade to serialize it:

\begin{lstlisting}[language=python]
@proof.attest
def attested_step(model, optimizer, batch):
    inputs, targets = batch
    logits = proof.forward(model, inputs)
    loss = F.cross_entropy(logits, targets)
    proof.backward(loss)
    optimizer.step()
    optimizer.zero_grad()
\end{lstlisting}

\paragraph{Training} Finally, users can train normally, either by invoking the attested step or using an equivalent fast path that produces numerically identical results:

\begin{lstlisting}[language=python]
for inputs, targets in training_batches:
    batch = (inputs, targets)
    with proof.transaction(batch=batch):
        logits = model(inputs)
        loss = F.cross_entropy(logits, targets)
        proof.backward(loss)
        optimizer.step()
        optimizer.zero_grad()
\end{lstlisting}

\section{Protocol Design Details}
\label{appx:sec:protocol-design}
\subsection{Commitment Hashes} \label{sec:commit:hashes}
Throughout the prose, in order to ensure protection from adaptive sampling attacks, we will use a series of $\mathcal{H}_{*}$ hash operations. Unless otherwise specified, $\mathcal{H}_{*}$ is meant to be understood as a cryptographically secure hash against $\mathcal{A}$ (in practice, BLAKE3), which contains a session-dependent stamp and secret material computed by $B_{\text{att}}$ not revealed to $\mathcal{A}$. Inputs use an injective, domain-separated serialization $\operatorname{ser}$. This per-hash stamping derives its security from the per-update ratcheted secret key in \cref{appx:sec:rachet}; after each commitment is stamped, the key ratchet state is advanced and cannot be re-used to evaluate another commitment.

\subsection{Batch Fingerprinting}
\label{appx:sec:batch-fingerprinting}

We fix a choice of $L$, the minimum window size for which a sequence of tokens is considered as a ``sample.'' Using this window size, we then perform a polynomial rolling hash:

Let $d_i$ be encoded as a sequence of tokens $d_i = \qty(d_i^{(1)} \dots d_i^{(m)})$. We compute two values: a secure commitment $C_{d}$ and a matching hash $F_{d}$:

\begin{equation}
C_d\qty(d_i)=\mathcal H_d\qty(\operatorname{ser}\qty(d_i)),\qquad
F_d\qty(d_i)=\qty(H_{\text{poly}}\qty(d_i^{(1)}),\dots,H_{\text{poly}}\qty(d_i^{(m-L+1)})).
\end{equation} 

Where each sub-hash is:
\begin{equation}
H_{\text{poly}}\qty(d_i^{(r)}) = d_i^{(r)} a^{L-1} + \dots + d_i^{(r+L-1)} a^{0} \ \text{mod}\ p.
\end{equation} 
for large prime $p$ and base $a$. Committing $C_d$ allows us to ensure future sampling is a secure function of the training data, while $F_d$ allows matching rolling hashes.

\subsection{Non-Geometric Optimizer State Hash}
\label{sec:sketch:optim}
Let's again partition the optimizer state as a series of leaves: $s_{i} := \mqty(s_{i}^{(1)}, \dots, s_{i}^{(m)})$.

We define an injective view $W: \mathbb{F}_{\mathrm{adm}} \to \qty{0,\dots,2^{31}-1}$ which re-represents each floating-point value as an integer. To see that the representation is unique, we restrict $\text{dom}\ W$ to be standard hardware-resident finite-precision floats. Then, following a $T$-round BLAKE-like construction \citep{aumasson2010blake}, for each boundary $s_i$, leaf $j$, round $t$, and coordinate $r$, $B_{\mathrm{att}}$ samples independent hidden $a_{j,t,r}\sim\operatorname{Unif}\qty(\mathbb Z_{2^{32}}^\times)$ and $b_{j,t,r}\sim\operatorname{Unif}\qty(\mathbb Z_{2^{32}})$. That is, we fix odd $a$ and not-necessarily-odd $b$ codewords as a random function of the boundary state. The same family hashes the two copies of $s_i$, remains unsampled until both are fixed, and is then discarded. We compute:
\begin{equation}
h_{j,t}\qty(s_i) = \sum_{r=1}^{|s_{i}^{(j)}|} a_{j,t,r} \qty [W\qty (s_{i}^{(j)})_{r} \oplus \qty(r a_{j,t,r} + b_{j,t,r})]\ \qty(\text{mod}\ 2^{32}),
\end{equation} 

where $\oplus$ is the bitwise XOR operation. Finally, we define:

\begin{equation}
C_{s}\qty(s_i) = \mathcal{H}_{\text{state}}\qty(h_{1,1}\qty(s_i) \dots h_{m, T}\qty(s_i))\ \qty(\text{mod}\ 2^{256})
\end{equation} 

where $\mathcal{H}_{\text{state}}$ is any standard cryptographic Merkle hash chain. 
\subsection{Record Signature Ratcheting}
\label{appx:sec:rachet}

The tape being produced establishes a signed hash chain of such record payloads, specifically that:

\begin{equation}
\tau_{i} = f_{\text{chain}_{i}}\qty(\tau_{i-1}, \bar{\tau}_{i}) = \qty(\bar{\tau}_{i}, \mathcal{H}_{\text{record}}\qty(\tau_{i-1}), \mathrm{vk}_{i+1}, \sigma_{i})
\end{equation} 

where $\bar\tau_{i}$ is the current payload of the record (one of $R, M, u_i$); $\mathcal{H}_{\text{record}}$ is a cryptographic hash of records. Finally, $\sigma_{i}$ is the signature of the current record, which is given as:

\begin{equation}
  \sigma_{i} = \operatorname{Sign}_{\mathrm{sk}_{i}}\qty(\mathcal{H}_{\text{sig}}\qty(\operatorname{ser}\qty(\bar\tau_i,\mathcal{H}_{\text{record}}\qty(\tau_{i-1}),\mathrm{vk}_{i+1})))
\end{equation} 

where $\operatorname{Sign}_{\mathrm{sk}_i}$ is a cryptographic signature using secret signing key $\mathrm{sk}_{i}$, verified by public key $\mathrm{vk}_i$. We expire each secret key $\mathrm{sk}_{i}$ after signing only one step.
\begin{equation}
  \qty(\mathrm{sk}_{i+1},\mathrm{vk}_{i+1})=\operatorname{KeyGen}\qty(\mathcal H_{\mathrm{ratchet}}\qty(\mathrm{sk}_i)).
\end{equation}

\subsection{Root Record Structure}
\label{appx:sec:root}
The root record $R$ is responsible for establishing the cryptographic primitives necessary for the resulting signing and attestation. It contains a session-specific nonce, timestamp, prover build information, and protocol configuration. In order to bootstrap trust in the system, the session public key is bound by an attestation under a stable attester key from which trust is derived. Each subsequent tape record will be signed by a series of keys from the ratcheted session signing state. Logically, let:

\begin{equation}
R_0=\qty(M,\nu,c,\mathcal H_{\mathrm{init}}\qty(\operatorname{ser}\qty(\theta_0,s_0)),\mathcal H_{\mathrm{func}}\qty(f),\mathrm{vk}_0),\qquad
R=\qty(R_0,\operatorname{Sign}_{\mathrm{sk}_{\mathrm{root}}}\qty(\mathcal H_{\mathrm{sig}}\qty(\operatorname{ser}\qty(R_0)))).
\end{equation} 

where $\operatorname{Sign}_{\mathrm{sk}_{\mathrm{root}}}$ is a signature made by a stable root-of-trust key, and $M$ is session-specific metadata.

\subsection{Measurement Record Structure}
\label{appx:sec:measure}
In order to ensure trust in the authenticity of the recording binary $B$, we additionally add trusted self-measurements as a record in the binary. This involves two measurement words $m_1, m_2$ and a cryptographic hash $M=\qty(\mathcal{H}_{\text{measure}}\qty(m_{1}), \mathcal{H}_{\text{measure}}\qty(m_{2}))$, where $m_{1}$ attests the binary portion of the system and $m_{2}$ attests any running code by both script analysis and live interpreter freezing. Details on the second step are given in \cref{sec:systems:invoke}.

\subsection{Update Record Structure} \label{sec:appx:update}
Let $C_{\theta}: \mathbb{R}^{|\theta|} \to \mathbb{R}, C_{s}: \mathbb{R}^{|s|} \to \mathbb{R}$ be the parameter and state sketching functions described in \cref{sec:sketch}, and let $C_{d}$ be the data commitment described in \cref{sec:karp}. Furthermore, let $\text{PRNG}: \{0,1\}^{*} \to [0, 1]$ be a uniform probabilistic random number generator.

The logical ``update record'' consists of five sub-records $u_{i}^{(1)} \dots u_{i}^{(5)}$, constructed chronologically.

\paragraph{Pre-State Commitment} Given $\theta_{i}, s_{i}$ and the user's declared batch $d_{i}$, we compute and record parameter, state, and data sketches from before the update occurs: $u_{i}^{(1)} = \qty(C_{\theta}\qty(\theta_{i}), C_{s}\qty(s_{i}), C_{d}\qty(d_{i}),F_d\qty(d_i))$.

\paragraph{Post-State Commitment} $A$ computes the training algorithm following $\tilde{f}\qty(\theta_{i}, s_{i}, d_{i})$, yielding $\theta_{i+1}, s_{i+1}$. We record these outputs, as declared by $A$, as well: $u_{i}^{(2)} = \qty(C_{\theta}\qty(\theta_{i+1}), C_{s}\qty(s_{i+1}))$.

\paragraph{Deterministic Sampling} Let $\mathcal{H}_{\text{record}}$ be a cryptographically secure hash; we follow a known $\text{PRNG}$ function and compute $u_{i}^{(3)} = p_{\text{challenge}} = \text{PRNG}\qty[\mathcal{H}_{\text{record}}\qty(\nu,\tau_{i-1}, u_{i}^{(1)}, u_{i}^{(2)})]$. This is our probabilistic sampling decision: \textit{we only challenge $f \cong^{?} \tilde{f}$ when $p_{\text{challenge}} \leq p_{\text{check}}$}. This ordering provides the unpredictability property of challenges: $A$ cannot cheat on only what will not be challenged, and $\mathcal{H}_{\text{record}}$ is ratcheted upon computation to prevent continuous search.

\paragraph{Challenge Replay} If our records are challenged in $u_{i}^{(3)}$, $B$ restores the committed function $f$ and computes $f\qty(\theta_{i}, s_{i}, d_{i}) = \theta_{i+1}', s_{i+1}'$. We compare $C_\theta\qty(\theta_{i+1}')$ and the full $s_{i+1}'$ against committed values, then challenge one slot $x$ of $\theta_{i+1}'$ and commit $u_{i}^{(4)} = \qty(\theta_{i+1}^{(x)}, \theta_{i+1}^{(x)'})$. Note that while the state buffer and parameter sketch are locally \textit{checked}, only one parameter is \textit{attested} on tape.

\paragraph{Checkpoint Commitment} If $A$ wants to emit a checkpoint, we transfer it to host, compute $h_i^{\text{ckpt}}=\mathcal{H}_{\text{ckpt}}\qty(\operatorname{ser}\qty(\nu,\tau_{i-1},u_i^{(1)},\dots,u_i^{(4)},\theta_{i+1}))$, and commit it. We additionally sample $q$ positions in the model weight and commit them. That is, we sample $x_{1} \dots x_{q} = \text{PRNG}\qty[\mathcal{H}_{\text{record}}\qty(h_i^{\text{ckpt}})]$ and commit $u_{i}^{(5)} = \qty(h_i^{\text{ckpt}},\theta_{i+1}^{(x_{1})} \dots \theta_{i+1}^{(x_{q})})$.



\section{Verification Procedure}\label{sec:param:bind}
\label{appx:sec:verification-procedure}

Our verification algorithm checks a given transcript $\tau$, initial state $(\theta_0,s_0)$, final checkpoint $\theta_{f}$, predeclared training algorithm $\mathcal{H}_{\text{func}}\qty(f)$, and nonce $\nu$ for their declared structure, and ensures that every update is structurally sound.

\paragraph{Attestation Checking $V_{\text{sig}}$} We examine the binding of $B_{\text{att}}$ to $B_{\text{host}}$. We check that the root session key derivation in the root record $R$ is valid from the published $B_{\text{att}}$ program. We \textbf{reject} if initial attestations are not signed from known builds, or if (in the zkVM case) initial session materials are not constructed from a known algorithm. Otherwise, we \textbf{continue}.

\paragraph{Parameter Binding $V_{\text{bind}}$} We then check if the checkpoint we want to attest $\theta_{f}$ actually belongs to the tape $\tau$. To do this, we extract $Q_1 \dots Q_{m} \in \tau$, the parameter sketch described in \cref{sec:sketch:param}, and apply them again to $\theta_{f}$. After hashing, this produces a single digest $C_{\theta}\qty(\theta_{f})$ which we search for in the tape. We also check that the checkpoint's full commitment $h_i^{\text{ckpt}}$ recomputes from $\theta_f$, that it's in the tape, and that the sampled values match. \textbf{Reject} if any check fails; \textbf{continue} otherwise.

\paragraph{Replay Binding $V_{\text{func}}$} We then check that the checked transition was consistent with the declared transition. It is not important that this transition is disclosed to the verifier, as long as it is consistent. We first \textbf{reject} if $\mathcal{H}_{\text{func}}\qty(f) \not\in \tau$, and then check the attestation bound to $\mathcal{H}_{\text{func}}\qty(f)$ in the tape; we \textbf{reject} if the attestation is invalid (e.g., it's not derived from the session key, or---in the non-whitebox binding setting---the zkVM execution tape is invalid). Otherwise, we \textbf{continue}.

\paragraph{Tape Integrity Checking $V_{\text{tape}}$} We next check the integrity of the tape by monitoring each of its transitions. Specifically, we follow the ratcheting design of \cref{appx:sec:rachet} to ensure the integrity of each chain link, bind the first record to $(\nu,\theta_0,s_0)$, check adjacent pre/post parameter and optimizer-state commitments, and check the declared tape contents' hash against the signed hash material. We \textbf{reject} if any check fails, and \textbf{continue} otherwise.

\paragraph{Sampled Replay Checking $V_{\text{replay}}$} Lastly, we check that the challenged value and PRNG sampling were done correctly for the challenged updates. Specifically, we verify the challenge derivation, replayed parameter sketch, full optimizer state, and sampled parameter value, and verify that the application of $f$ matches the same hash as the one bound in $\mathcal H_{\text{func}}$. \textbf{Reject} if replays are incorrect or sampling is not followed; \textbf{accept} otherwise.

\section{Complete Proof}
\label{appx:sec:proofs}
We now give the full proof of the soundness statement made in \cref{sec:security}. 
\subsection{Formalism}
\label{sec:formal}

For an $n$-step training run, $A$ produces parameter states $\theta_{0} \dots \theta_{n}$ and optimizer states (and metadata, such as LR schedule) $s_{0} \dots s_{n}$ using data samples $d_{0} \dots d_{n-1}$. $s_i$ or $d_i$ will also include computed or supplied random seeds, respectively. $A$ thus follows a deterministic procedure $\tilde{f}: \qty(\Theta, S, D) \to \qty(\Theta, S)$ from a combination of parameters $\Theta$, state $S$, and data $D$ to the corresponding next parameter and state.

Confirmation $c \in C$ includes $p_{\text{check}},p_{\text{sketch}},T,q$ (check probability, sketch size, hashing rounds, and number of checkpoint sketch vectors) and the cryptographic parameters. Our goal is to develop $\mathcal{P}: \qty(\Theta, \Theta, S, S, D, C) \to T$ such that $\mathcal{P}\qty(\theta_{i}, \theta_{i+1}, s_{i}, s_{i+1}, d_{i}, c) = \tau_{i}$ which maps each of the transitions above into linkages in a tape format $\tau_{i} \in T$. 

Define now also an honest training procedure: $f: \qty(\Theta, S, D) \to \qty(\Theta, S)$. We ask that $f$ has bounded description length $K\qty(f) \leq B$, where $K$ is Kolmogorov complexity, such that any information used in training in practice comes from the data source itself. In practice, we can instantiate this as source-code length as lowered into an intermediate format; see \cref{sec:systems:seralize}.

Let $\mathcal{H}_{\text{func}}\qty(f)$ be a one-way commitment of $f$ (e.g., hashing the code of $f$). We define a verifier $V\qty(\tau, c, \nu, \theta_{0}, s_{0}, \mathcal{H}_{\text{func}}\qty(f), \theta_f) \in \qty {0,1}$ which should accept with the following properties. Here, $\theta_f \in \tau$ means that checkpoint $\theta_f$ is bound to some checkpoint release record in $\tau$.

\subsection{Threat Model Guarantees}
\label{appx:sec:threat-model-guarantees}
Let's start by implementing our threat model as a series of axioms which we appeal to. We first operationalize that $A$ is a bounded polytime adversary for which our hashes are sound:

\label{ax:hashes}
\begin{axiom}[cryptographic primitives]
  Every $\mathcal{H}_{*}$ used in \cref{sec:algo} admits a cryptographic hash, so for all machines $\mathcal A$ that are under our scope:

  \paragraph{Hashes are Secure} \begin{equation}
\Pr\qty[\mathcal A\text{ finds $x,y$ with $\operatorname{poly}\qty(\lambda)$ executions s.t. }\mathcal H_{*}\qty(x)=\mathcal H_{*}\qty(y),x\neq y]\leq\operatorname{negl}\qty(\lambda).
  \end{equation} 

\paragraph{Bindings are Secure}
  $\operatorname{Sign}$ is EUF-CMA secure, commitments are binding, and $B_{\mathrm{att}}$ is sound under the same bound.
\end{axiom}

where $\lambda$ is a standard security parameter, $\operatorname{poly}\qty(\lambda)$ is any polynomial function of $\lambda$, and $\operatorname{negl}\qty(\lambda)$ is a negligibly small function of $\lambda$.

Next, we describe how we ensure the freshness of each run:

\label{ax:challenge}
\begin{axiom}[challenge freshness]
  Conditioned on each position in the tape, sampling a coordinate position to check under $p_{\mathrm{check}}$ is uniform and independent. $V$ registers one attempt for each nonce $\nu$; aborting a run results in rejection.
\end{axiom}

This nonce setup means that an attacker cannot grind our randomly sampled checks end-to-end until a single run results in no cheating samples being detected. In practice, this risk is prevented by a random whitebox-generated nonce, and the practical fact that $p_{\mathrm{check}}=0.01$ means it takes in expectation $\frac{1}{0.01} = 100$ tries of full training to get one valid sample, which is impractical.

We now turn to a fact that allows us to maintain buffer identity between updates:

\label{ax:param-continuity}
\begin{axiom}[attested parameter continuity]
  $B_{\mathrm{att}}$ enforces post/pre parameter-buffer identity across adjacent updates, and $V$ validates its attestation.
\end{axiom}

We also need to formalize that $A$ cannot grind through every possible training trajectory using unbounded time to find a cheating trajectory, which supports our bounded-computation threat model:

\begin{axiom}[prover is limited]
  All machines $\mathcal A$ under our scope are probabilistic polynomial-time in $\lambda$ and their input length.
\end{axiom}
where $\lambda$ is a security parameter.

\subsection{Optimizer State Binding}
\label{appx:sec:optimizer-state-binding}

Let's begin now by discussing some statements about $C_{s}$ and show that this simple reduction admits an amplifiable, non-geometric fingerprint. We will restrict our statement to one leaf of the hash in \cref{sec:sketch:optim}, and appeal to \cref{ax:hashes} for the soundness of the final reduction.

First, we will use the following lemma about the structure of the fold:

\begin{lemma}[optimizer binding fold collision] \label{lem:osb:fold}
  For distinct leaf values $x,y\in\mathbb F_{\mathrm{adm}}^N$, one round of the fold in \cref{sec:sketch:optim} satisfies:
  \begin{equation}
    \Pr\qty[h_{j,t}\qty(x)=h_{j,t}\qty(y)]\leq\frac12.
\end{equation} 
\end{lemma}

\begin{proof} \label{appx:sec:osb:osf}
  We desire that one round collides with probability at most $1/2$.

  For simplicity, choose a coordinate $r$ such that $W\qty(x_r)\neq W\qty(y_r)$, and let $v$ be the location of their smallest differing bit. Since $W\qty(x_r),W\qty(y_r)<2^{31}$, we have $v\leq30$.

  Consider any $z,x,y\in\mathbb Z_{2^{32}}$, and the bitwise XOR operation $\oplus$. The following statement is true:
  \begin{equation}\label{eqn:proof:osb:plus}
    v_2\qty(\qty[x\oplus z]-\qty[y\oplus z])=v_2\qty(x-y).
  \end{equation} 
  This is because XOR by the same $z$ preserves the smallest differing bit.

  Using this fact, condition on all other coordinates and let $z=ra_{j,t,r}+b_{j,t,r}$. Since $b_{j,t,r}$ is uniform, $z$ is uniform and independent of the odd $a_{j,t,r}$. Call
  \begin{equation}
    d=\qty[W\qty(x_r)\oplus z]-\qty[W\qty(y_r)\oplus z].
  \end{equation} 
  Applying \cref{eqn:proof:osb:plus}, $v_2\qty(d)=v$. Multiplication by uniform odd $a_{j,t,r}$ is therefore uniform over the $2^{31-v}$ residues of valuation $v$.

  Finally, the conditioned sum can equal zero for at most one such residue, so:
  \begin{equation}
    \Pr\qty[h_{j,t}\qty(x)=h_{j,t}\qty(y)]\leq2^{-(31-v)}\leq\frac12.
  \end{equation} 

  In particular, this gives the desired per-round collision bound.
\end{proof}

The proof idea here is that finite computation precision (e.g., mod $2^{32}$) preserves the least significant bit of difference between two inputs, while the random odd multiplier shuffles the resulting residue class uniformly. We give the full proof in \cref{appx:sec:osb:osf}. Using this fact, we can state an important result about our binding.

That is, each round of our binding admits a ``uniform shuffling'' within a residue class keyed by $a,b$. We will now use this fact to finish our claim about the security of our binding:

\begin{mdframed}
\begin{lemma}[Security of $C_{s}$]\label{lem:osb}
  For distinct $x,y\in\mathbb F_{\mathrm{adm}}^N$,
  \begin{equation}
\Pr\qty [C_{s}\qty(x) = C_{s}\qty(y)] \leq 2^{-T} + \operatorname{negl}\qty(\lambda).
  \end{equation}
\end{lemma}
\end{mdframed}

\begin{proof}
  Choose a differing leaf $j$. From \cref{lem:osb:fold}, each round of $h_j$ collides with probability at most $1/2$. $T$ such rounds, recall, with fresh $a,b$ commitments, collide with probability at most $2^{-T}$. Finally, we appeal to \cref{ax:hashes}, which guarantees final fold security up to a negligible bound $\operatorname{negl}\qty(\lambda)$ for a security parameter $\lambda$, to obtain the final collision bound.
\end{proof}

In practice, any precision-preserving collision usually degrades model output quality.

We have now shown that $C_{s}$ is a sound commitment scheme for optimizer state parameters. We will move on to showing the same for $C_{\theta}$ next, with the additional constraint that the initial sketching component must also be \textit{geometric}.

\subsection{Parameter Sketching}
\label{appx:sec:parameter-sketching}

Let's now prove a similar bound for the fast geometric sketch described in \cref{sec:sketch:param}. Unlike the non-geometric hash of the optimizer state, this sketch must be computable after the fact by the verifier independently, so it cannot derive its power by a series of constantly refreshed, state-dependent shuffled $a,b$, and must survive floating-point precision differences between hardware. However, we prove the security of our system below using exact arithmetic, and leave to empirical work in \cref{sec:expr} how this sketch performs in practice.

Recall that we sketch by a series of projections using orthonormal matrices $Q_j \in \mathbb{R}^{n_j \times l_j}$. As before, let's consider only a single leaf and call it $Q \in \mathbb{R}^{n \times l}$ for some parameter-matrix domain $n$ and $l=\left\lceil p_{\mathrm{sketch}}n\right\rceil$, and look to \cref{ax:hashes} for the security of the final fold.

\begin{lemma}[$Q^{T}$ matrices are full spark]\label{lem:psb:spark}
  Every set of $l$ columns in $Q^{T}$ is linearly independent. That is, $\text{spark}\qty(Q^{T}) = l+1$ almost surely.
\end{lemma}

\begin{proof}
  $Q$ is generated from sampled Gaussian values; thus every $l \times l$ minor is nonsingular almost surely. Thus each choice of $l$ columns of $Q^T$ is linearly independent, and thus you must choose $l+1$ columns for linear dependence. Orthogonalization preserves rank. Thus the spark of $Q^{T}$ is $l+1$ almost surely.
\end{proof}

Intuitively, \cref{lem:psb:spark} can help us use the linearly independent projections to ``pin down'' degrees of freedom from the weight matrices we are attesting through multiplication. We can formalize this by observing the behavior of $Q$ applied to two differing matrices.

\begin{lemma}[amount of cheating is lower-bounded]\label{lem:psb:bound}
  Let $\theta \neq \tilde{\theta}$ be rows of a single weight matrix, and call $\delta = \tilde{\theta} - \theta$. $\delta Q = 0, \delta\neq 0$ implies $\norm{\delta}_{0} \geq l +1$.
\end{lemma}

\begin{proof}
  We have $\delta\neq 0, \delta Q = 0$. Suppose for the sake of contradiction that $\exists \delta$ such that $\norm{\delta}_{0} < l+1$. Then this implies that there exists a linearly dependent set of rows in $Q$ of size $\leq l$. But $Q$ is full spark, reaching a contradiction.
\end{proof}

Now, fix $p_{\text{check}}$ as the probability that an update is challenged, followed by one uniformly sampled parameter, as in \cref{sec:update}. Let $\Theta \neq \tilde{\Theta}$ be two distinct parameter trees, and let $\Delta = \tilde{\Theta} - \Theta$.
Further define:
\begin{equation}
  G\qty(\Delta) = \min\qty(|\Theta|,\sum_{j}\sum_{r}\bold{1}_{\Delta_{r}^{(j)} \neq 0}\qty(l_j+1)),
\end{equation}
where $r$ indexes matrix rows. Using this $G$ construction, we are ready to claim the size of cheating we can prevent with our sketch:

\begin{mdframed}
\begin{lemma}[Security of $C_{\theta}$]\label{lem:psb:boundsurviv}
  The chance that $\Theta \neq \tilde{\Theta}$ survives unnoticed is upper-bounded:
  \begin{equation}
    \Pr\qty [\text{undetected}] \leq 1- p_{\text{check}} \qty(\frac{G\qty(\Delta)}{|\Theta|}) + \operatorname{negl}\qty(\lambda).
  \end{equation}
\end{lemma}
\end{mdframed}

\begin{proof}[Proof of \cref{lem:psb:boundsurviv}] \label{appx:sec:psb:boundsurviv}
  First, observe that the sketch not detecting a difference implies either no cheating or
  \begin{equation}
    \norm{\Delta}_{0} \geq G\qty(\Delta)
  \end{equation} 
  by considering each row of $\Delta$ and applying \cref{lem:psb:bound}. Consider now the single-parameter challenge scheme described in \cref{sec:update}. For every ``checked'' parameter, we obtain exact assurance for the computation of the parameter. Thus, on a challenged update, a cheating parameter going unnoticed has to survive two events: $E_{n}$ that it has to lie in the nullspace of the sketch, and $E_{c}$, that it goes unchecked. Specifically:
  \begin{equation}
 \Pr\qty[\text{undetected}]\leq\qty(1-p_{\text{check}})+p_{\text{check}}\Pr\qty(E_c\mid E_n)\Pr\qty(E_n)+\operatorname{negl}\qty(\lambda)
  \end{equation} 
  with security parameter of the final hash $\lambda$.

  Since we are working against an adaptive $\mathcal{A}$, who may explicitly optimize for $\Pr\qty(E_{n})$ (e.g., all her parameters lie in the sketch null space), let's set the worst case:
    \begin{equation}
\Pr\qty[\text{undetected}]\leq\qty(1-p_{\text{check}})+p_{\text{check}}\Pr\qty(E_c\mid E_n)+\operatorname{negl}\qty(\lambda).
  \end{equation} 

  Now begin by observing that:
  \begin{equation}
\Pr\qty(E_{c}\mid E_n)=1-\frac{\norm{\Delta}_{0}}{|\Theta|}
  \end{equation} 
  that is, a uniformly sampled check will only catch a parameter if they differ in the delta, and we are sampling the checked parameter uniformly; note that we only challenge an update with chance $p_{\text{check}}$.
  Now apply the fact that we are conditioned on $E_{n}$, that is, we know from above that $\norm{\Delta}_{0} \geq G\qty(\Delta)$ if the parameters lie in the sketch null space; we can write:
    \begin{equation}
\Pr\qty(E_{c}\mid E_n) \leq 1-\frac{G\qty(\Delta)}{|\Theta|}.
  \end{equation} 

  Finally, replacing the LHS and carrying the inequality through:
  \begin{equation}
    \Pr\qty[\text{undetected}] \leq 1-p_{\text{check}}\qty(\frac{G\qty(\Delta)}{|\Theta|}) + \operatorname{negl}\qty(\lambda).
  \end{equation} 

\end{proof}

The proof idea here is to apply \cref{lem:psb:bound} to every row of each parameter matrix, and to observe that either the sketch catches a cheat or the cheat lies in the sketch null space, meaning the sampling has more chances to catch the cheat.  A full proof of this is given in \cref{appx:sec:psb:boundsurviv}.

\subsection{Tape Integrity}
\label{appx:sec:tape-integrity}

Now that we have proofs for each of our basic objects, we will now complete the proof by making some statements about the proof tape appealing to usual cryptographic primitives, before moving on to our final theorems.

Let's first make the statement which makes explicit the ``chained'' structure of our tape:

\begin{lemma}[state continuity]\label{lem:chain}
  Verifier \textbf{accept} means that for every update $u_i$, $i\in\{0,\dots,n-2\}$ for $n$ training steps, we have:
  \begin{align}
&\theta^{\text{post}}\qty(u_i)=\theta^{\text{pre}}\qty(u_{i+1}) \\
&C_{\theta}^{\text{post}}\qty(u_{i}) = C_{\theta}^{\text{pre}}\qty(u_{i+1}) \\
&C_{s}^{\text{post}}\qty(u_{i}) = C_{s}^{\text{pre}}\qty(u_{i+1}) 
  \end{align} 
  that is, each update uses parameter material that the previous update generated, and optimizer-state material up to hash collision.
\end{lemma}
\begin{proof}
  $V_{\text{sig}}$ validates \cref{ax:param-continuity}, and $V_{\text{tape}}$ as described in \cref{sec:param:bind} explicitly checks the commitments.
  By \cref{lem:osb}, optimizer-state substitution succeeds with probability at most $2^{-T}+\operatorname{negl}\qty(\lambda)$.
\end{proof}

At this point, we can formalize that every form of cheating within a transition is caught by at least one previously-bounded mechanism. Let $f$ be the attested transition and $\tilde{f}$ be a potentially-malicious transition that is executed, and let $\theta', s' = f\qty(\theta, s,d)$, and $\tilde{\theta}, \tilde{s} = \tilde{f}\qty(\theta, s,d)$. Define $\Delta = \theta' - \tilde{\theta}$, $g_{\min}=\min_j\qty(l_j+1)$, and
\begin{align}
\epsilon_{\theta}\qty(\Delta)&=\min\qty(1,1-p_{\text{check}}\frac{G\qty(\Delta)}{|\Theta|}+\operatorname{negl}\qty(\lambda)),\\
\epsilon_s&=1-p_{\text{check}},\\
\bar\epsilon\qty(c)&=\max\qty(\min\qty(1,1-p_{\text{check}}\frac{\min\qty(g_{\min},|\Theta|)}{|\Theta|}+\operatorname{negl}\qty(\lambda)),1-p_{\text{check}}).
\end{align}

\begin{lemma}[all cheats are bounded]\label{lem:bound}
  \begin{equation}
    \Pr\qty[V\text{ accepts update }i\mid\theta'\neq\tilde\theta\vee s'\neq\tilde s]\leq\bar\epsilon\qty(c).
  \end{equation} 
\end{lemma}
\begin{proof}
  We consider casework. For the case where $\theta' \neq \tilde{\theta} \vee s' \neq \tilde{s}$ (i.e., there was a functional cheat), only one of three things may occur:
  \paragraph{$\theta' \neq \tilde{\theta}, s' = \tilde{s}$} In this case, our parameter update can fall only into one of two mutually exclusive cases: $C_{\theta}(\tilde{\theta}) \neq C_{\theta}(\theta')$ or $C_{\theta}(\tilde{\theta}) = C_{\theta}(\theta')$. In either case, we directly appeal to \cref{lem:psb:boundsurviv} to see that the chance of undetection is at most $\epsilon_{\theta}\qty(\Delta)\leq\bar\epsilon\qty(c)$. That is, $\Pr\qty(\text{accept}\mid\theta'\neq\tilde\theta,s'=\tilde s)\leq\epsilon_\theta\qty(\Delta)$.

  \paragraph{$\theta' = \tilde{\theta}, s' \neq \tilde{s}$} Similar to before, our state update can fall into one of two mutually exclusive cases: $C_{s}(\tilde{s}) \neq C_{s}(s')$ or $C_{s}(\tilde{s}) = C_{s}(s')$. In either case, exact challenged replay gives a chance of undetection at most $\epsilon_s\leq\bar\epsilon\qty(c)$. That is, $\Pr\qty(\text{accept}\mid\theta'=\tilde\theta,s'\neq\tilde s)\leq\epsilon_s$.

  \paragraph{$\theta' \neq \tilde{\theta}, s' \neq \tilde{s}$} The probability of acceptance for cheating in both is at most the minimum of the two as the first-detected failure would result in rejection. That is, $\Pr\qty(\text{accept}\mid\theta'\neq\tilde\theta,s'\neq\tilde s)\leq\min\qty(\epsilon_s,\epsilon_\theta\qty(\Delta))$.

  Thus the disjunction of all three cases is upper-bounded by $\max\qty(\epsilon_{\theta}\qty(\Delta),\epsilon_s)\leq\bar\epsilon\qty(c)$.
\end{proof}

With this lemma, we are now ready to claim our desired statement about tape soundness.
\begin{mdframed}
\begin{theorem}[Tape update checking is sound] \label{prop:psb:boundsurviv-ckp}
  \begin{equation}
\Pr\qty[V\qty(\tau,c,\nu,\theta_0,s_0,\mathcal H_{\text{func}}\qty(f),\theta_f)=1\mid I_f\qty(\tau)\geq k]\leq\varepsilon_1\qty(k,c).
  \end{equation} 
\end{theorem}
\end{mdframed}
\begin{proof}
  Notice that given \cref{lem:chain}, a cheating tape is recognized at least as much as each of its individual cheats are recognized. Let $J=\qty{i:\text{update }i\text{ is invalid}}$, and let $E_i$ be the event that cheating survives step $i$. By the chain rule, we have that:
  \begin{equation}
    \Pr\qty[\bigcap_{i\in J}E_i]=\prod_{i\in J}\Pr\qty[E_i\mid E_{<i}].
  \end{equation} 
  Applying \cref{ax:challenge,lem:bound}, we have:
  \begin{equation}
    \Pr\qty[\bigcap_{i\in J}E_i]\leq\prod_{i\in J}\bar\epsilon\qty(c)=\bar\epsilon\qty(c)^{|J|}.
  \end{equation} 
  Finally, our verifier will reject all faulty updates it finds, so it will not reject only when updates slip past unnoticed, so:
  \begin{equation}
    \Pr\qty[\text{accept}]\leq\bar\epsilon\qty(c)^{|J|}\leq\bar\epsilon\qty(c)^k.
  \end{equation} 
  Thus, define $\varepsilon_1\qty(k,c)=\bar\epsilon\qty(c)^k$, which gives our desired cheating rate.
\end{proof}

\subsection{Checkpoint Integrity}
\label{appx:sec:checkpoint-integrity}

We can now prove the other side of the soundness claim by essentially repeating the argument given in \cref{lem:psb:boundsurviv}. Specifically, we have:

\begin{corollary}[Security of final checkpoint $C_{\theta}\qty(\theta_{f})$]\label{coro:psb:boundsurviv-ckp}
  For $\Delta=\tilde\Theta-\Theta$ fixed by the full checkpoint commitment, the chance that $\Theta \neq \tilde{\Theta}$ survives unnoticed with $q$ committed value samples is upper-bounded:
  \begin{equation}
    \Pr\qty[C_{\theta}\qty(\Theta)=C_{\theta}\qty(\tilde{\Theta})\wedge\text{samples agree}]\leq\qty(1-\frac{G\qty(\Delta)}{|\Theta|})^q+\operatorname{negl}\qty(\lambda).
  \end{equation}
  This follows from exactly the argument given in \cref{appx:sec:psb:boundsurviv} for \cref{lem:psb:boundsurviv}, except there are now $q$ independently sampled checked locations after the full commitment. Thus, the joint probability of all $q$ challenges missing is the product of the probabilities of each challenge missing.
\end{corollary}

From this, we can then give the second part of our soundness result:

\begin{mdframed}
\begin{theorem}[Checkpoint commitment is sound] \label{prop:ckp}
  \begin{equation}
    \Pr\qty[V\qty(\tau,c,\nu,\theta_0,s_0,\mathcal H_{\text{func}}\qty(f),\theta_f)=1\mid I_f\qty(\tau)=0,\theta_f\notin\tau]\leq\varepsilon_2\qty(c).
  \end{equation} 
\end{theorem}
\end{mdframed}

\begin{proof}
  Write $h_i^{\mathrm{ckpt}}\qty(\theta)$ for the checkpoint hash in \cref{sec:update} with the same release prefix. For $\theta_{f}\not\in\tau$, there are two cases. Either $h_i^{\mathrm{ckpt}}\qty(\theta_f)\notin\tau$, in which case $V_{\text{bind}}$ rejects with probability 1, or $h_i^{\mathrm{ckpt}}\qty(\theta_f)\in\tau$ yet $\theta_f$ is not part of training. This implies that $\exists\theta_f'\in\tau$ such that $h_i^{\mathrm{ckpt}}\qty(\theta_f')=h_i^{\mathrm{ckpt}}\qty(\theta_f)$. If so, \cref{ax:hashes} gives directly a negligible bound. Define
  \begin{equation}
    \varepsilon_2\qty(c)=\operatorname{negl}\qty(\lambda).
  \end{equation} 
  This gives our desired cheating rate.
\end{proof}

\subsection{Additional Properties}
\label{appx:sec:additional-properties}
We state, but do not show, four additional claims about our tape. The first two are informally true by inspection, and the first three are evaluated empirically in \cref{sec:expr}.

\paragraph{Streamable Tape} We make no attempts at specifying the structure of the data, so we simply ask that our tape is linearly sized $|\tau|  \leq \kappa\qty(c) n$, where $\kappa\qty(c) > 0$ and that verification happens in one pass with a working set that is $O\qty(1)$ to training updates $n$.

\paragraph{Easily Checkable Tape} We require that verification costs substantially less work than fully reproducing the whole training run. We leave this statement to be less formal than others to enable independence between different platforms and lowering, and with the expectation that ``substantially less'' is nevertheless still an asymptotically linear function of the original training length.

\label{appx:sec:logit-checkpointing}
\paragraph{Logit Checkpointing} During our update function, the proof tape can optionally store the full top-k logit values for each input in the batch. This allows for additional security because of the results shown in \cref{fig:sec:structured-attack} and \cref{fig:sec:optimized-attack}, where we show that even if an attacker is able to find a parameter update that evades the sketch, it is still unlikely to do so without meaningfully changing training distribution performance.

\paragraph{Initial-State Attacks} To enable support for tasks like fine-tuning, we do not make any claims about the initial state parameters other than fine-tuning, and assume that $A$ cannot sneak unbounded undeclared information through their initial state that's not visible by inspection since the verifier also has access to $\theta_{0}, s_{0}$.

Since we want our proofs to be general over the model architecture itself, we do not assume that the model has any specific behavior with respect to its inputs, and thus leave the theoretic bounds of these statements unqualified.

\section{Performant Systems Design}
\label{appx:sec:systems-design}
We now discuss some aspects of the system design that are critical for both the security and the practical performance results given in \cref{sec:res}.

\subsection{Update Program Serialization}
\label{sec:systems:seralize}

Soundness of the algorithm depends on attesting and running a pure update function $f\qty(\theta, s, d) \to  \qty(\theta', s')$. Enforcing this over typical ML code is non-trivial because there's no obvious decomposition of arbitrary user code without explicit side effects.

We therefore must design a stable, serializable closure over any user's choice of $f$. This closure is important for two reasons: first, we need to make sure that side-channel information cannot be passed through a side-channel into $f$, thus defeating the purpose of $d$ being the only source of data; second, we need to bound the size of $|f|$ by explicit measurement (i.e., an $f$ that closes over all pretraining data does not need $d$ to add information to the model).

To achieve this, we use the insight that \textit{accelerator program serialization already provides a source of functional environment closure} which we can serialize and reuse. Our Python facade observes any user code which ran during one attested training step, and then calls the standard serialization pathway on the resulting graph which the user returns to us. This graph serialization is already available as \texttt{jax.export.export} for Jax and \texttt{torch.aot\_compile.export} for Torch.

This setup requires only the user to disclose what tensor corresponds to which logical elements (e.g., tensor, optimizer state, etc.), and assumes all other inputs and outputs are static information to be closed on. The resulting compiled binary is then measured for size, signed, and identically rehydrated for use whenever computation on $f$ is required.

\subsection{Interleaved State Offloading and Recovery}
\label{sec:systems:offloading}

In order to minimize memory usage for any given flow control configuration, we make two interleaved offloading decisions that minimize and hide the overhead of memory copying and loading.

For any given update $u_{i}$, there are five data structures worth considering: $\theta_{i}, \theta_{i+1}, \Delta_{i}, s_{i}, s_{i+1}$. In naive training, we keep three buffers \texttt{T}, \texttt{D} and \texttt{S}, whereby we compute logically \texttt{*T, *S = update(*T, *S)} and use \texttt{D} to store $\Delta_{i}$ intermediates. This setup ``donates'' buffers for $\theta_{i},s_{i}$ (\texttt{T} and \texttt{S} respectively) for use by $\theta_{i+1}, s_{i+1}$ respectively.

However, naively implementing our algorithm requires \textit{all five} buffers to be kept in memory, in case an update is challenged and thus we must run the committed update function on the attested path using the original inputs.

In order to prevent this, our pipelined update program eagerly dispatches asynchronous copies of the optimizer state and allows usual buffer donation logic to free the optimizer state buffers after copying. Since this copying is interleaved with the forward pass, it usually produces little to no bubbles under the same HBM operation. If an update is challenged, the already-donated $\theta_{i}$ parameters are backed out via $\theta_{i}=\theta_{i+1}-\Delta_{i}$ while the cached state is loaded from host memory. We illustrate the design of this offloading scheme in the pseudocode presented in \cref{fig:offload}.

\begin{figure}[h]
\begin{lstlisting}[language=c]
  void update(buf *T, buf *S, scratch_buf *D) {
    host_buf *S_host = async_copy(S); // eager copying

    *D = compute_deltas(T, S)
    *S = compute_next_state(T, S)
    *T = apply_update(T, D);

    int sketch_host = sketch(T,S,D);

    if (challenged_p(T,S)) { // non hot-path
      S_host -> await();
      *S = copy(S_host);
      S -> await() // expensive
      *T = *T - *D; 

      assert sketch_host  == check(T,S,D);
    } 
  }
\end{lstlisting}
\caption{A logical illustration of offloading design. We don't wait for state offload to finish if an update turns out to be unchallenged. This is sound since we always wait when we need the value in the host buffer; host-side memory can otherwise be without coherence without affecting correctness.}
\label{fig:offload}
\end{figure}

\subsection{Eliding Offload Computation}
\label{appx:sec:eliding-offload}
One other technique that can be used to trade off throughput and memory involves not offloading the optimizer state during every single step, and instead doing so at a fixed interval. Following \citet{choi2023tools}, if a step that is not offloaded eventually becomes challenged, we can simply replay training from the last offloaded step until the challenged step for evaluation. The pace for this offload can be tuned optimally based on empirical measurements of the interconnect speed and compute throughput (it's better to offload more frequently if the interconnect is fast).

Specifically, let $p=p_{\text{check}}$ be the probability we sample an update, $N$ be the number of updates between offloads, and $r=\frac{1}{N}$. Let offloading cost $\gamma$ and replay cost $\epsilon$ per transition. The amortized per-update cost is:

\begin{equation}
g\qty(r) = \gamma r + \frac{1}{2} \epsilon p\qty(\frac{1}{r}-1)
\end{equation} 

This expression is strictly convex and has the solution:

\begin{equation}
r^{*} = \sqrt{\frac{\epsilon p}{2\gamma + \epsilon p}}
\end{equation} 

We can then use this rate to tune the offload frequency to minimize the amortized cost of offloading and replaying. 

\subsection{Invoked Python Sealing}
\label{sec:systems:invoke}

Our attested worker must on occasion call Python APIs to manipulate opaque tensors obtained by Python. In order to prevent memory movement or re-encoding of these objects, we share the same interpreter as the user program. This design introduces a meaningful new security surface: whereby user code runs interleaved with the attested Python calls, thus surfacing potential vulnerabilities where users modify or patch runtime state of attested programs. We provide a best-effort closure which prevents easy exploitation of this vulnerability. Note that a non-whitebox system would use proof of computation techniques to replace these (light) functions.

In order to prevent state modifications, we protect six types of Python state:

\paragraph{Fingerprint} The identity of each attested function is copied and hashed by observing its bytecode. During each call, we repeat this procedure to check that the fingerprints did not change.

\paragraph{Environment} For every attested function, its object, \texttt{\_\_globals\_\_}, value defaults, keyword defaults, any closed state, and any modules are copied using \texttt{deepcopy}. We also check during execution that none of its arguments has a mutable input type.

\paragraph{Code constraints} We inspect each attested function's Python code at each runtime to identify global reads. Explicit reads to dictionaries, lists, sets, byte arrays are rejected at runtime. Specific reflexive operations are explicitly rejected during attested runtime: \texttt{eval}, \texttt{exec}, \texttt{globals}, \texttt{getattr}, \texttt{setattr}, \texttt{open}, \texttt{\_\_import\_\_}.

\paragraph{Bytecode constraints} We inspect each attested function's bytecode to prevent global reads, rejecting functions that invoke \texttt{STORE\_GLOBAL}, \texttt{STORE\_ATTR}, \texttt{STORE\_SUBSCR}, \texttt{DELETE\_*}, \texttt{IMPORT\_*}.

\paragraph{Modules} Each imported function from third-party modules is also sealed recursively but without AST walking. To ensure this is sound, we ensure that our system is the first import before imported modules (we enforce ordering by poisoning imported module context), since we control the correct implementations of each attested function.

\paragraph{Torch/Jax JIT} Just-in-time compiled functions are sealed prior to compilation, and we build a new JIT wrapper and request compilation ourselves. The resulting compiled binary is sealed and attested similar to \cref{sec:systems:seralize}.

Importantly, we control the original implementations of each attested function. So the overall goal here is \textit{not} containment of arbitrary user programs, but to ensure that adversarial modifications to our implementations are caught by a variety of heuristics.

\subsection{Flow Control and Accelerator Memory Management}
\label{sec:systems:flow-control}

Since our driver and proof work are done asynchronously to training processes, we must ensure that buffers needed for the proof are not freed too early while ensuring that we do not create multiple memory buffers in-flight (e.g., having $s_{i}, s_{i+1}$ in flight at the same time, which would quadruple the HBM memory usage). To achieve this, we use a back-pressured flow control system to decide when to admit more training work given the lifetime of the previous update.

This system enforces two invariants: 1) the maximum number of updates scheduled by the driver is always $\leq u_{\max}$, a parameter set by the user; 2) a sampled check's buffers will never be overwritten by training.

To do this, we maintain a state machine for each update $u_{i}$ following three possibilities: \textbf{active}, meaning computation for evidence is still in place, \textbf{done}, meaning computation for evidence is no longer in place but the update is still memory-resident (e.g., a pointer to its resources is still alive), \textbf{retired}, meaning resources have been freed, or \textbf{crashed}, meaning computation failed.

An update moves from \textbf{active} to \textbf{done} when its resulting $\theta_{i+1}, s_{i+1}$ has been produced on the accelerator.

An update moves from \textbf{done} to \textbf{retired} when the broker signals it is not challenging the update, or the challenge is finished. 

An update moves from \textbf{active} to \textbf{crashed} when it crashes.

There are two potential configurations for the flow control system. The \texttt{ALWAYS} configuration dispatches a new update, $u_{i+1}$, if the current updates in flight $|u_{\text{active}}| + |u_{\text{done}}| < u_{\max}$. The \texttt{UNCHECKED} configuration dispatches $u_{i+1}$ if $|u_{\text{active}}| + |u_{\text{done}}| < u_{\max}$ and there are no updates that are currently being challenged. In both configurations, the flow control crashes to poison the rest of training when any update $u_{i}$ returns the crashed state.

This system ensures that users can trade off throughput with memory usage, and ensures that failures are coordinated between the three subsystems and surfaced eventually to the user.

\subsection{Additional Notes on Execution Order and Pipelining}\label{sec:systems:timing}
At a high level, our system uses a standard API, broker, worker pattern, where a facade Python API calls a Rust task broker that maintains a queue of tasks that it dispatches to a single, cooperatively scheduled worker that releases the Python Global Interpreter Lock (GIL) during user code execution. Notice especially that the driver retains all the \textit{logical} control for the training process, but the worker dispatches the actual update.

\Cref{tab:system-timeline} provides a timeline of the execution model of our system. This design allows the worker/accelerator pair to become desynced from driver and broker computations. Specifically, the hot path for a given update (e.g., the \textit{unchallenged update computation}) can be pipelined back-to-back without waiting for the sampling decision or cryptography (broker action) or user verification code (driver action) to finish. Users can dispatch multiple updates which will be fused by the worker to be back-to-back. 

Two implementation details make worker pipelining possible: $\theta_{i}, s_{i}$ are copied for sketching before the forward pass starts, meaning the HBM load for the $\tilde{f}$ forward pass can also be used for offloading, reducing memory bandwidth contention; furthermore, we donate $\theta_{i}$ immediately after usage, such that $\theta_{i+1}$ can be written in-place. To recover $\theta_{i}$ for sampled checks, we simply subtract the not-yet-donated update buffer (since the sampling decision returns before the backwards pass). We describe these mechanisms in detail in \cref{sec:systems:offloading}.

\section{Additional Evaluations}
\label{appx:sec:additional-evaluations}

\paragraph{Sketch Geometry} We additionally specifically evaluate the security of our checkpoint provenance sketch \cref{sec:sketch} by perturbing a final checkpoint randomly and measuring the degree to which such a perturbation is detected by the sketch, validating the discussion in \cref{sec:sketch}. In order to do this, we sweep across different locations in a final checkpoint and inject random changes across $1 \times 10^{-{5}} - 1\times 10^{-1}$ magnitudes. We measure whether our system detects this checkpoint corruption upon validation, and if not, what the relative effect of the perturbation on the output logits is on orthonormal input vectors. We aim to demonstrate that our sketch can detect any perturbations that meaningfully affect the output logits. In \cref{fig:sec:structured-attack}, we show that our sketch is able to detect any structured parameter corruption that meaningfully affects the output logits as measured by output KL divergence.

\begin{figure}[h]
  \centering
  \includegraphics[width=\linewidth]{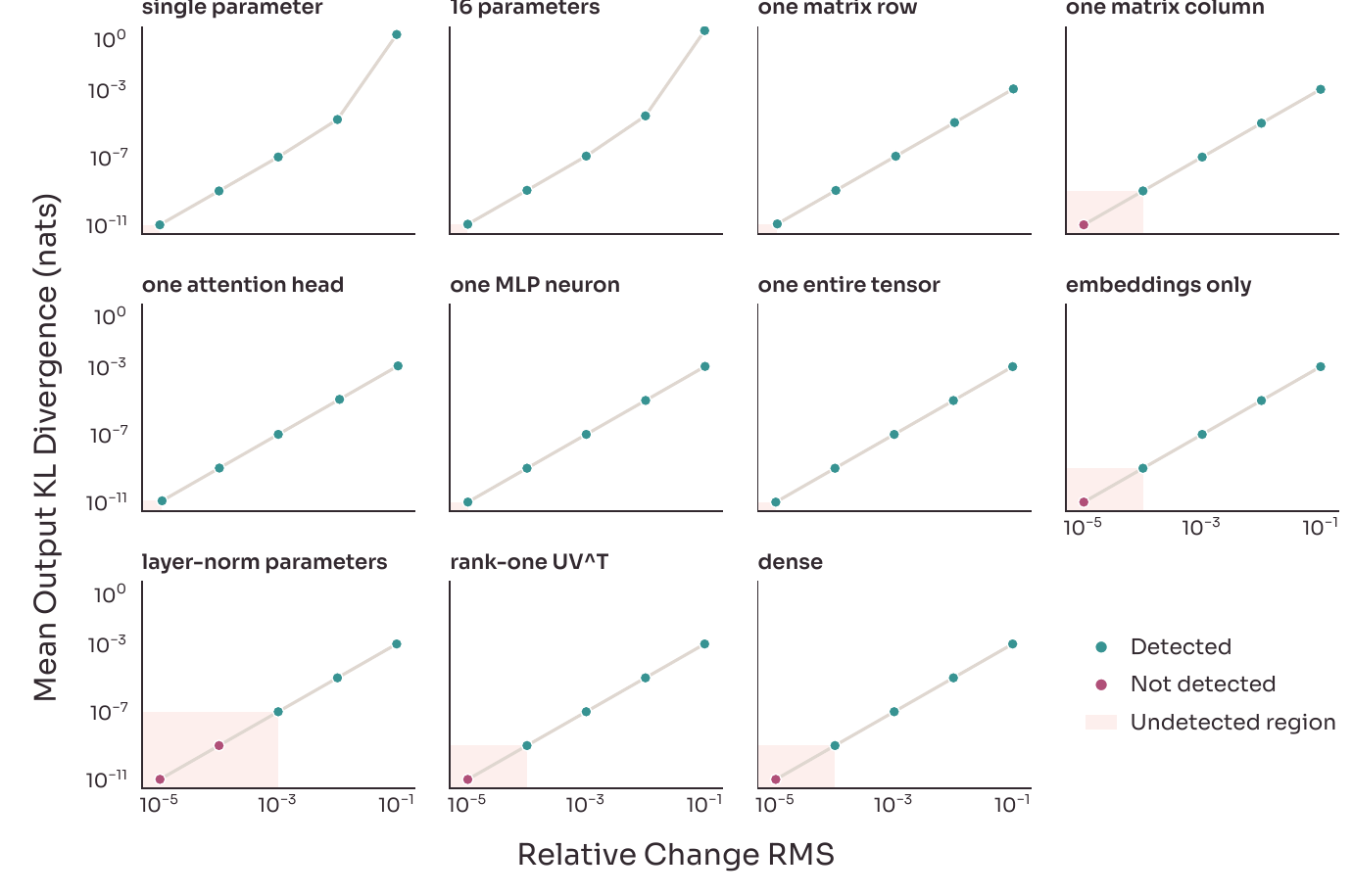}
  \caption{\textbf{Cheating only in limited subspaces is also reliably detected by our sketch}. Structured attacks on a 100M parameter model during a single transition, sweeping across different relative change RMS magnitude. Structured perturbations are applied across different components of the model, and we measure the next-token KL divergence induced on 6 orthonormal input embeddings. ``Detected'' means that the sketch outputs differ by the $1\times 10^{-5}$ floating point difference threshold. Note that this same sketch would secure provenance and transition claims, and undetected changes result in no top-k changes and vanishingly small logit changes.}
\label{fig:sec:structured-attack}
\end{figure}

\paragraph{No Pipeline Stalls} In \cref{fig:sec:single-gpu}, we show that our system does not meaningfully stall the training pipeline. Specifically, we do not significantly decrease token throughput, and only drop the GPU MFU by $<3\%$ in the worst case, since multi-GPU training would incur additional communication stalls beyond the single-GPU case.

\begin{figure}[h]
  \centering
  \includegraphics[width=\linewidth]{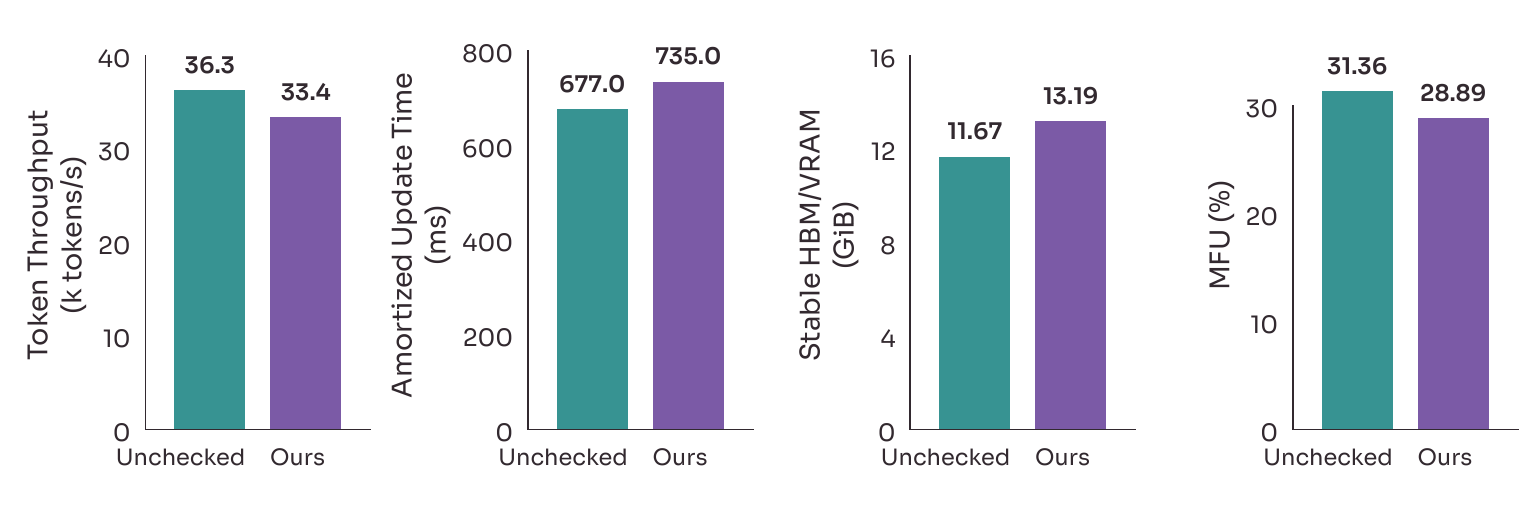}
  \caption{\textbf{In compute-bound settings, our system can still maintain high throughput without a drastic cost to MFU}. Saturated single-GPU overheads for our method with AdamW. \textbf{1}: token throughput overhead; \textbf{2}: amortized per-step training time; \textbf{3}: mean HBM resident size (bytes in use); \textbf{4}: MFU impact as measured by published theoretical BF16 Tensor Core GEMMs. 1.03B dense GPT-style model on single-GPU with sizes scaled for compute (memory) roofline saturation; profiling training steady-state at $1\%$ check probability. Measurement conducted on a single NVIDIA H100 NVL72.}
\label{fig:sec:single-gpu}
\end{figure}

\paragraph{Log-Linear Model Size Scaling} In \cref{fig:sec:scaling}, we show that our system's amortized update time and mean HBM usage scale roughly log-linearly with model parameters. This result gives us an indication of the potential scaling pattern of our system and its practicality at even larger scales.

\begin{figure}[h] 
  \centering
  \includegraphics[width=0.49\linewidth]{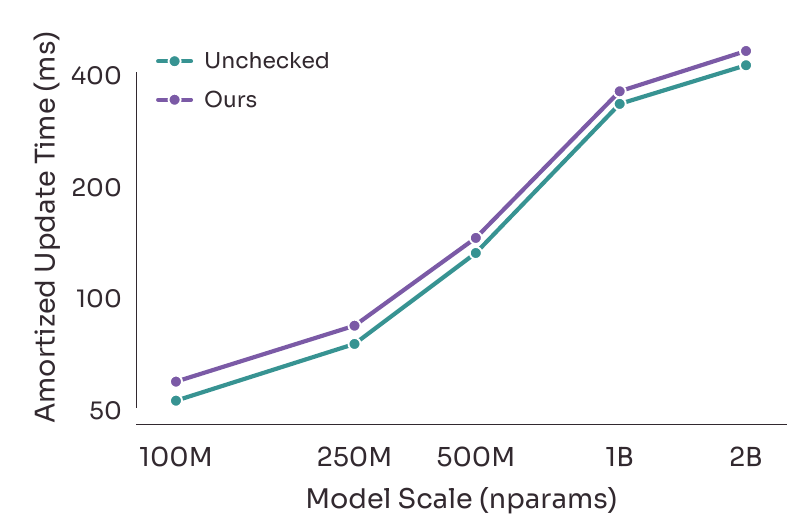}
  \includegraphics[width=0.49\linewidth]{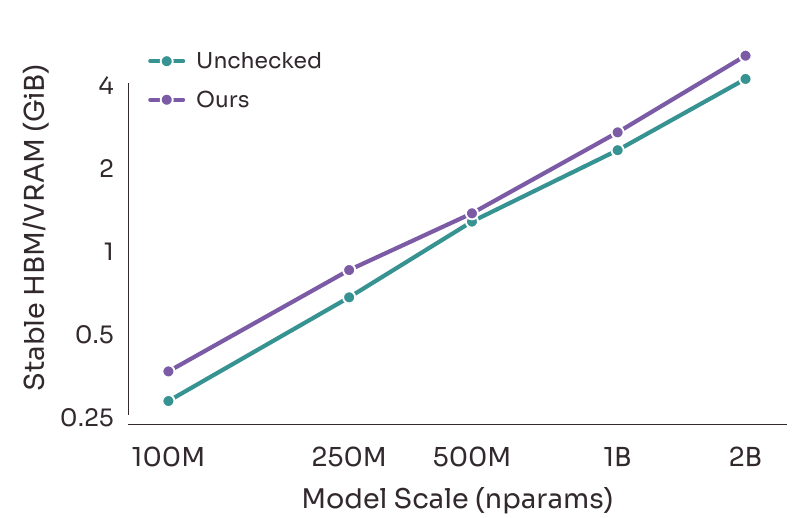}
  \caption{\textbf{Our performance cost scales log-log-linearly with model size, as is similar for model training itself}. DDP time and memory overheads for applying our method with SGD across a variety of scales. \textbf{Left}: amortized per-step training time; \textbf{Right}: mean HBM resident size (bytes in use) per GPU. Single seed, conducted over 2 NVIDIA H100 NVL72. Axes are log-log scale.}
\label{fig:sec:scaling}
\end{figure}

\paragraph{No Influence on Training} In \cref{fig:appx:loss-worse}, left, we replay a very small toy run of a 100M parameter model under DDP with and without our system. We note that training loss is floating-point identical on the same hardware, since neither our procedure nor our implementation changes the training procedure.

\paragraph{Update Time Comparison} In \cref{fig:appx:loss-worse}, right, we benchmark the time spent on a single update of our system under 100M and break down the time spent on each component of the update. We identify that, as expected, state offloading and flow control (i.e., tape worker overhead) consist of the overheads measured in our system.

\begin{figure} \label{fig:appx:loss-worse}
  \centering
  \includegraphics[width=0.495\linewidth]{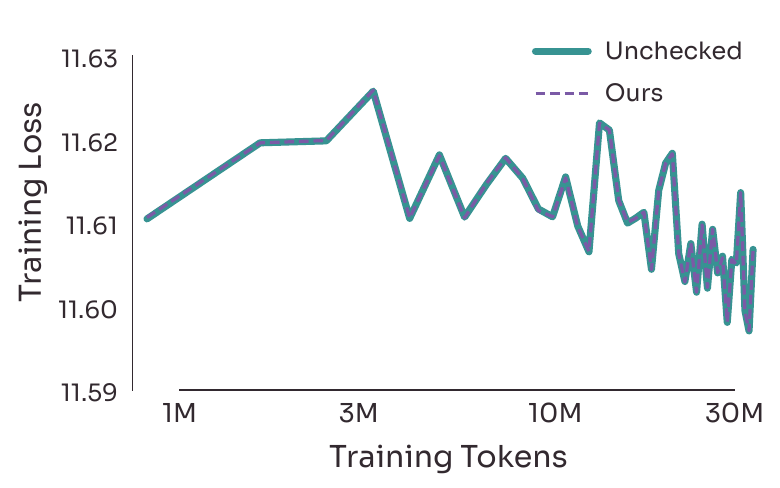}
  \includegraphics[width=0.495\linewidth]{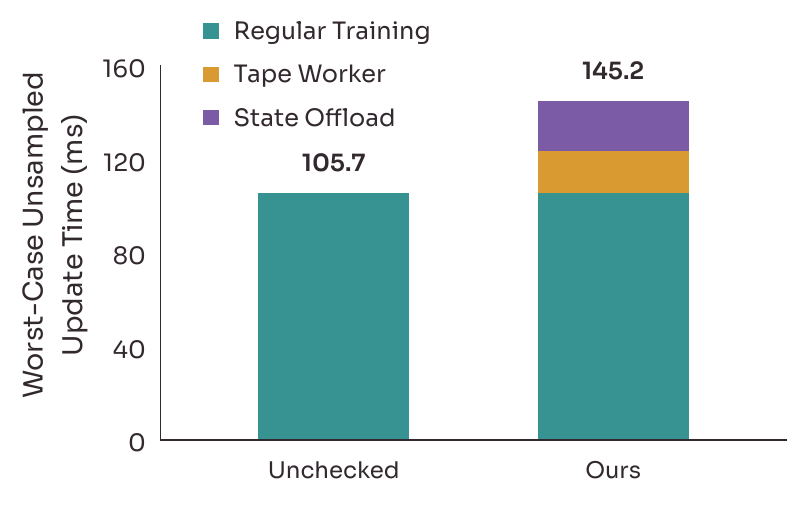}
  \caption{\textbf{Left}: \textbf{Our system does not change the actual training}. Toy example of the very beginning of training of a 100M model. Our system does not change the training procedure and indeed the training is bitwise reproducible on the same hardware with and without our system. \textbf{Right}: \textbf{Our overheads are caused primarily by tape write backpressure and state offloading costs}. Breakdown of the structure of one update at 100M. Note that the update overhead comprises state offloading (amortized) time and certificate worker write overhead (instantiated as the update being held by flow control).}
\end{figure}

\paragraph{Red Teaming} Although we did not conduct a full professional security audit, we examined the implementation of our white-box cryptographic setup with a minimal red-teaming exercise. Note that we do not claim the white-box cryptographic setup is exhaustively secure, but that it is a reasonable boundary with which to set up the rest of the system. Specifically, we prompted \texttt{Kimi K3} in the official Kimi Code harness with a built wheel of our project and the instruction (\texttt{Hello! I have a wheel in this folder. It creates a training data verification transcript. Your job is to figure out a way to 1) reverse engineer its algorithm, and 2) find out a way to make a valid, "wrong" transcript. By wrong, I don't just mean cheat the verifier. it means: a neural network that behaves "reasonably correctly" (i.e., not obviously wrong), but includes training data undeclared by the transcript, and for which the verifier holding the checkpoint and certificate cannot distinguish based on re-running the algorithmic hashes provided in the transcript.}). After 8 hours of continuous work, the agent was unable to produce a valid transcript that enabled it to mark the prompted goal as complete.

\end{document}